\def\year{2018}\relax
\documentclass[letterpaper]{article} %DO NOT CHANGE THIS
\usepackage{aaai18}  %Required
\usepackage{amsmath}\allowdisplaybreaks
\usepackage{epsfig,epstopdf,algorithmic,algorithm,amsfonts,url,textcomp,amsthm,multirow,moresize}
\usepackage{graphicx}
\usepackage[tableposition=top]{caption}
\usepackage{subcaption}
\usepackage{graphicx}
\usepackage[title]{appendix}
\usepackage{natbib}
\newtheorem{assumption}{Assumption}
\newtheorem{remark}{Remark}
\usepackage{enumitem}
\newtheorem{theorem}{Theorem}
\newtheorem{lemma}{Lemma}
\newtheorem{definition}{Definition}

\newcommand*\samethanks[1][\value{footnote}]{\footnotemark[#1]}

\begin{document}
\title{Inexact Proximal Gradient Methods for  Non-convex \\ and Non-smooth Optimization}
 \author{Bin Gu\samethanks[1],
De Wang\samethanks[2],
Zhouyuan Huo\samethanks[1], Heng Huang\samethanks[1]\\
\samethanks[1]Department of Electrical \& Computer Engineering, University of Pittsburgh, USA\\
\samethanks[2]Dept. of Computer Science and Engineering, University of Texas at Arlington,  USA \\
 big10@pitt.edu, wangdelp@gmail.com, zhouyuan.huo@pitt.edu, heng.huang@pitt.edu}

\maketitle
\begin{abstract}
In  machine learning research, the proximal gradient methods are popular for solving various optimization problems with  non-smooth  regularization. Inexact proximal gradient methods  are extremely important when   exactly solving  the proximal operator  is  time-consuming, or the proximal operator  does not have an analytic solution. However,  existing inexact proximal  gradient methods only consider  convex problems. The knowledge of inexact proximal gradient methods in the non-convex setting is very limited. % Moreover, for some machine learning models, there is still no proposed solver for exactly solving the  proximal operator.
 To address this challenge, in this paper, we first propose three  inexact proximal gradient algorithms, including the basic version and Nesterov's accelerated version. After that, we provide the theoretical analysis to the basic and Nesterov's accelerated versions.  The theoretical results show that our  inexact proximal gradient algorithms can have the same  convergence rates as the ones of exact proximal gradient algorithms  in the non-convex setting.
 Finally, we show the   applications of our inexact proximal gradient algorithms on  three representative non-convex learning problems. All experimental results confirm the  superiority of our new inexact proximal gradient algorithms. %The applications on robust OSCAR and link prediction show that,  our inexact proximal gradient algorithms are  significantly faster than the  exact ones. The application on robust Trace Lasso fills the vacancy that there is no proximal gradient algorithm for Trace Lasso.
\end{abstract}

\section{Introduction}
Many machine learning problems involve non-smooth  regularization, such as the machine learning models with a variety of sparsity-inducing penalties \citep{bach2012optimization}. Thus, efficiently solving the optimization problem with non-smooth  regularization is important for many machine learning applications. In this paper, we focus on the  optimization problem of machine learning model with non-smooth  regularization as the following form:
   \begin{eqnarray}\label{formulation1}
\min_{x\in \mathbb{R}^N} f(x) = g(x)+ h(x)
\end{eqnarray}
where  $g: \mathbb{R}^N \to \mathbb{R}$ corresponding to the empire risk is smooth and possibly non-convex, and $h: \mathbb{R}^N \to \mathbb{R}$ corresponding to the regularization term is  non-smooth and possibly non-convex.

Proximal gradient methods are popular for solving various optimization problems with  non-smooth  regularization. The pivotal step of the  proximal gradient method is to  solve the proximal operator as following.
\begin{eqnarray} \label{section1_1.1}
\textrm{Prox}_{\gamma h}(y) = {\arg \min}_{x\in \mathbb{R}^N} \frac{1}{2 \gamma}  \left \| x-y \right \|^2 + h(x)
\end{eqnarray}
where $\gamma $ is the  stepsize.\footnote{The  stepsize $\gamma $ is  set manually or automatically determined by a backtracking line-search procedure \citep{beck2009fast}.}
If the function $h(x)$  is simple enough, we can obtain the solution of the proximal operator analytically. For example, if $h(x)=\| x \|_1$, the solution of the proximal operator can be obtained by a shrinkage thresholding operator \citep{beck2009fast}. If the function $h(x)$  is complex such that the corresponding proximal operator does not have an analytic solution, a specific algorithm should be designed for solving the proximal operator. For example, if $h(x)=\| x \|_1+ c\sum_{i<j} \max \{ | x_i | , |x_j  | \}$ as used in  OSCAR \citep{zhong2012efficient}  for the sparse regression  with automatic feature grouping,  \cite{zhong2012efficient} proposed an iterative group merging  algorithm for exactly solving the proximal operator.

However,   it would be expensive to solve  the proximal operators  when the function $h(x)$  is complex.  Once again, take OSCAR as an example, when the data is with high dimensionality (empirically larger than 1,000), the iterative group merging  algorithm   would become very inefficient.
Even worse, there would be no solver for exactly solving the  proximal operators when the function $h(x)$ is over complex. For example,
  \cite{grave2011trace} proposed the  trace Lasso norm to take into account the correlation of the design matrix to stabilize the estimation in regression. However, due to the complexity of trace Lasso norm,   there   still have no  solver for solving the corresponding proximal operator, to the best of our knowledge.

  \begin{table*}[htb]
%  %\vspace*{-6pt}
\small
 \center
 \caption{Representative (exact and inexact) proximal gradient  algorithms. (C and NC are the  abbreviations of convex and non-convex respectively.) } %\vspace{-0.3cm} \setlength{\tabcolsep}{4mm}
\begin{tabular}{c|c|c|c|c|c}
\hline
\textbf{Algorithm}  & \textbf{Proximal} & \textbf{Accelerated}  &  $g(x)$ &    $h(x)$  &  \textbf{Reference}  \\ \hline
PG+APG  &  Exact & Yes & C   & C  & \cite{beck2009fast} \\
APG  & Exact & Yes & C+NC    &  C & \cite{ghadimi2016accelerated} \\
PG  & Exact & No & NC   & NC  & \cite{boct2016inertial} \\
APG  & Exact & Yes & C+NC  & C+NC  & \cite{li2015accelerated} \\
 \hline
IPG+AIPG  & Inexact & Yes & C  & C  & \cite{schmidt2011convergence} \\
AIFB (AIPG)  & Inexact & Yes & C  & C  & \cite{villa2013accelerated} \\
IPG+AIPG  & Inexact & Yes & C+NC   & C+NC  & Ours
\\ \hline
\end{tabular}
\label{table:methods}
  %\vspace*{-8pt}
\end{table*}

To address the above issues,    \cite{schmidt2011convergence} first proposed the inexact proximal gradient  methods (including  the basic version (IPG) and Nesterov's accelerated version (AIPG)), which  solves the proximal operator approximately (\emph{i.e.},  tolerating an error  in the calculation of the proximal operator). They proved that the inexact proximal gradient  methods can have  the
same convergence rates as the ones of exact proximal gradient
methods, provided
that the errors in computing the proximal operator decrease at appropriate rates.
Independently,   \cite{villa2013accelerated}  proposed   AIPG algorithm and proved the corresponding convergence rate. In the paper of \citep{villa2013accelerated}, they called AIPG as the inexact forward-backward
splitting method (AIFB) which  is well-known  in the field of signal processing. We summarize these works  in Table \ref{table:methods}.

From Table \ref{table:methods}, we find that the existing inexact proximal  gradient methods only consider  convex problems. However,  a lot of optimization problems in machine learning are non-convex. The non-convexity  originates either from the empirical risk function $g(x)$  or  the regularization function $h(x)$. First, we investigate the empirical risk function $g(x)$ (\emph{i.e.}, loss function). The  correntropy induced loss \citep{feng2015learning}   is widely used  for robust regression and classification, which is non-convex.  The  symmetric sigmoid loss on the unlabeled samples is used in semi-supervised SVM \citep{chapelle2006continuation}  which is non-convex. Second,  we investigate  the  regularization function $h(x)$. Capped-$l_1$ penalty  \citep{zhang2010analysis} is  used for  unbiased variable selection,  and the low rank constraint \citep{jain2010guaranteed} is widely used for the matrix completion. Both of these regularization functions are non-convex.  However, our knowledge of inexact proximal gradient methods  is very limited in the non-convex setting.

%\cite{tappenden2013inexact} proposed an inexact coordinate descent (ICD) method  for
%the convex problems in which the function $h(x)$ is assumed to be block separable.
%As mentioned previously, non-convex and non-smooth optimization is important  in machine learning and machine learning. Thus, it is highly desired to give a general inexact proximal gradient method with the Nesterov's acceleration which can handle non-convexity in the functions $g(x)$ and $h(x)$, similar to the   method of Li and Lin \cite{li2015accelerated}.

To address this challenge, in this paper, we first propose three  inexact proximal gradient algorithms, including the basic and Nesterov's accelerated versions, which can handle the non-convex problems.
 Then we give the theoretical analysis to the basic and Nesterov's accelerated versions. The theoretical results show that our  inexact proximal gradient algorithms can have the same  convergence rates as the ones of exact proximal gradient algorithms.
Finally, we provide the   applications of our inexact proximal gradient algorithms on three representative non-convex learning problems. The applications on robust OSCAR and link prediction show that,  our inexact proximal gradient algorithms could be  significantly faster than the  exact proximal gradient algorithms. The application on robust Trace Lasso fills the vacancy that there is no proximal gradient algorithm for Trace Lasso.

\noindent \textbf{Contributions.} The main contributions of this paper are summarized as follows:
\begin{enumerate}
 %\vspace*{-3pt}
\item We  first propose the  basic and accelerated inexact proximal gradient algorithms  with  rigorous convergence guarantees.
   Specifically, our inexact proximal gradient algorithms can have the same  convergence rates as the ones of exact proximal gradient algorithms in the non-convex setting.
 %\vspace*{-3pt}
 \item We provide the  applications of our inexact proximal gradient algorithms on three  representative non-convex  learning problems, \emph{i.e.},     robust OSCAR,  link prediction and robust Trace Lasso. The results confirm the  superiority of our inexact proximal gradient algorithms.
%\vspace*{-3pt}
\end{enumerate}

%We organize the rest of the paper as follows.  In Section  \ref{section_preliminary}, we give some preliminaries. In Section \ref{section_algorithm}, we propose our IPG and AIPG algorithms. In Section \ref{convergence_analysis}, we prove the  convergence rates of IPG and AIPG for the non-convex and non-smooth problems. In  Section \ref{nonmonotoneAIPG}, we provide a  nonmonotone  version of AIPG to speed up AIPG. In Section \ref{experiments}, we present the experimental results on three typical learning applications.  Finally, in Section \ref{conclusion}, we give some concluding remarks.
%------------------------------------------------------------------------
\section{Related Works}
Proximal gradient methods are one of the most important methods for solving various optimization problems with  non-smooth  regularization.  There have been a variety of exact proximal gradient methods.

Specifically, for convex problems, \cite{beck2009fast} proposed  basic proximal gradient (PG) method and  Nesterov's accelerated proximal gradient (APG) method. They  proved that PG has the convergence rate $O(\frac{1}{T})$, and APG has the  convergence rate $O(\frac{1}{T^2})$, where $T$ is the number of iterations. For non-convex problems,  \cite{ghadimi2016accelerated}  considered that only $g(x)$ could be non-convex, and proved that the convergence rate of APG method is $O(\frac{1}{T})$.  \cite{boct2016inertial} considered that  both of $g(x)$ and $h(x)$ could be non-convex, and proved the  convergence  of PG method.   \cite{li2015accelerated} considered that  both of $g(x)$ and $h(x)$ could be non-convex, and proved that the APG algorithm can converge in a finite number of iterations, in a  linear rate  or a sublinear rate (\emph{i.e.}, $O(\frac{1}{T})$) at different conditions. We  summarize these exact proximal gradient methods in Table \ref{table:methods}. % Table \ref{table:methods} shows that the  method of  \citep{li2015accelerated} is the more general   which can handle the non-convexity  either from  $g(x)$ or $h(x)$.

In addition to the above batch  exact proximal gradient  methods, there are  the  stochastic and online proximal gradient  methods \citep{duchi2009efficient,xiao2014proximal}. Because they  are beyond  the scope of this paper, we do not review them in this paper.

\section{Preliminaries}\label{section_preliminary}
In this section, we introduce  the  Lipschitz smooth, $\varepsilon$-approximate subdifferential and $\varepsilon$-approximate Kurdyka-{\L}ojasiewicz (KL) property, which are  critical to  the convergence analysis of our inexact proximal gradient  methods  in the non-convex setting.

\noindent \textbf{Lipschitz smooth:} \ \
 For the smooth functions  $g(x)$, we have the Lipschitz constant $L$ for $\nabla g(x)$ \citep{wood1996estimation} as following.
\begin{assumption}\label{NormalLipschitzconstant}
\label{definition1}
$L $ is the  Lipschitz constant for $\nabla g(x)$. Thus,  for all $ x$ and $y$, $L$-Lipschitz smooth  can be presented as
   \begin{eqnarray}\label{ass3}
\| \nabla g(x) - \nabla  g(y) \|  \leq L \|x - y \|
\end{eqnarray}
Equivalently, $L$-Lipschitz smooth  can  also be written as the formulation (\ref{coordinate_lipschitz_constant2}).
   \begin{eqnarray} \label{coordinate_lipschitz_constant2}
 g(x) \leq g(y) + \langle \nabla g(y) , x-y \rangle + \frac{L}{2}  \left \| x-y \right \|^2
\end{eqnarray}
\end{assumption}

\noindent \textbf{$\varepsilon$-approximate subdifferential:} \ \  Because inexact proximal gradient is used in this paper, an $\varepsilon$-approximate proximal operator may produce  an $\varepsilon$-approximate subdifferential. In the following, we give the definition of $\varepsilon$-approximate subdifferential \citep{bertsekas2003convex}  which will be used in the $\varepsilon$-approximate KL property (\emph{i.e.}, Definition \ref{definition1}).
\begin{definition}[$\varepsilon$-approximate subdifferential] \label{definition0.1}
Given a convex function $h(x): \mathbb{R}^N \mapsto \mathbb{R}$ and a positive scalar $\varepsilon$, the $\varepsilon$-approximate subdifferential
 of $h(x)$ at a point $x \in  \mathbb{R}^N$ (denoted as $\partial_{\varepsilon} h(x)$) is
 \begin{eqnarray} \label{definition_subgradient}
\partial_{\varepsilon} h(x) = \left  \{ d \in \mathbb{R}^N : h(y) \geq h(x) + \langle d  , y-x \rangle - \varepsilon  \right \}
\end{eqnarray}
\end{definition}
Based on Definition \ref{definition0.1}, if $d \in \partial_{\varepsilon} h(x)$, we say that $d$ is an $\varepsilon$-approximate subgradient of $h(x)$ at the point $x$.

\noindent \textbf{$\varepsilon$-approximate KL property:} \ \
  Originally, KL property is introduced to analyze the convergence rate of exact proximal gradient methods in the non-convex setting \citep{li2015accelerated,boct2016inertial}. Because this paper focuses on the inexact proximal gradient methods,  correspondingly we propose the $\varepsilon$-approximate KL property in Definition \ref{definition1}, where the function $dist(x,S)$ is defined by $dist(x,S) = \min_{y\in S}  \left  \| x - y  \right \|$,  and $S$ is a subset of $ \mathbb{R}^N$.
\begin{definition}[$\varepsilon$-KL property]
 \label{definition1}
A function $f(x) = g(x)+h(x): \mathbb{R}^N \rightarrow ( - \infty, + \infty]$ is said to have the $\varepsilon$-KL property at $\overline{u} \in  \{u \in \mathbb{R}^N : \nabla g(u) + \partial_{\varepsilon} h(u)) \neq \emptyset \}$, if there exists $\eta \in (0,+ \infty ]$, a neighborhood $U$ of $\overline{u}$ and a function $\varphi \in \Phi_{\eta}$, such that for all $u \in U \cap \{ u \in \mathbb{R}^N : f(\overline{u}) < f({u}) <  f(\overline{u})  + \eta \}$, the following inequality holds
\begin{eqnarray} \label{definition1_1.2}
\varphi'(f(u) - f(\overline{u})) dist (\textbf{0}, \nabla g(u) + \partial_{\varepsilon} h(u))) \geq 1
\end{eqnarray}
where $\Phi_{\eta}$ stands for a class of functions $\varphi : [0,\eta) \rightarrow \mathbb{R}^+$ satisfying:

 \begin{enumerate}
 \item $\varphi$ is concave and continuously differentiable function on $(0,\eta)$;
 \item $\varphi$ is continuous at $0$, $\varphi(0) = 0$;
 \item and $\varphi'(x) >0$, $\forall x \in (0,\eta)$.
 \end{enumerate}
% (1) $\varphi$ is concave and continuously differentiable function on $(0,\eta)$; (2) $\varphi$ is continuous at $0$, $\varphi(0) = 0$; and (3) $\varphi'(x) >0$, $\forall x \in (0,\eta)$.
\end{definition}
%\vspace{-0.2cm}
 \section{ Inexact Proximal
Gradient  Algorithms}\label{section_algorithm}
In this section, we first propose the  basic  inexact proximal gradient algorithm for the   non-convex optimization, and then propose two  accelerated inexact proximal gradient algorithms.
%\vspace{-0.2cm}
 \subsection{Basic Version}
As shown in Table \ref{table:methods}, for the convex  problems (\emph{i.e.}, both the functions $g(x)$ and $h(x)$ are convex),   \cite{schmidt2011convergence}  proposed a basic inexact proximal gradient (IPG)   method. We follow the   framework of IPG  in \citep{schmidt2011convergence}, and give our IPG algorithm for the   non-convex optimization  problems (\emph{i.e.}, either the function $g(x)$ or the function $h(x)$ is  non-convex).

 Specifically, our IPG algorithm is presented in Algorithm \ref{algorithm1}.  Similar with the exact  proximal gradient algorithm, the pivotal step of our IPG (\emph{i.e.}, Algorithm \ref{algorithm1}) is to compute an inexact proximal operator $x \in \textrm{Prox}^{\varepsilon}_{\gamma h}  ( y)$  as following.
 \begin{eqnarray}\label{section2_equ1}
&& x \in \textrm{Prox}^{\varepsilon}_{\gamma h}  ( y)
= \nonumber \left \{ z\in \mathbb{R}^N :   \frac{1}{2  \gamma} \left  \| z - y  \right \|^2 + h(z) \right .
  \\ && \nonumber \left . \quad \quad \quad \quad \leq \varepsilon + \min_{x}\frac{1}{2 \gamma} \left  \| x - y  \right \|^2 + h(x) \right \}
\end{eqnarray}
where $\varepsilon$ denotes an error in  the calculation of the proximal operator. As discussed in \citep{tappenden2013inexact}, there are several methods to compute  the inexact proximal operator. The most popular method is using a primal dual algorithm to control the dual gap \citep{bach2012optimization}. Based on the dual gap, we can strictly control the error in  the calculation of the proximal operator.

\begin{algorithm}
\renewcommand{\algorithmicrequire}{\textbf{Input:}}
\renewcommand{\algorithmicensure}{\textbf{Output:}}
\caption{Basic inexact proximal gradient method (IPG)}
\begin{algorithmic}[1]
\REQUIRE  $m$, error $\varepsilon_k$ ($k=1,\cdots,m$), stepsize $\gamma < \frac{1}{L}$.
\ENSURE $x_{m}$.
 \STATE  Initialize  $x_0 \in \mathbb{R}^d$.
\FOR{$k=1,\cdots,m$}

 \STATE  Compute $x_{k} \in \textrm{Prox}^{\varepsilon_k}_{\gamma h} \left (x_{k-1} - \gamma  \nabla g(x_{k-1})  \right )$. %\COMMENT{Atomic writing}

\ENDFOR
\end{algorithmic}
\label{algorithm1}
\end{algorithm}
%\vspace{-0.2cm}
\subsection{Accelerated Versions}
We first propose a   Nesterov's  accelerated inexact proximal gradient algorithm for  non-convex optimization, then give a nonmonotone  accelerated inexact proximal gradient algorithm.

As shown in Table \ref{table:methods}, for the convex optimization problems,  \cite{beck2009fast}  proposed a Nesterov's accelerated inexact proximal gradient (\emph{i.e.}, APG)   method, and   \cite{schmidt2011convergence}  proposed a Nesterov's accelerated inexact proximal gradient (\emph{i.e.}, AIPG)   method. Both of APG and AIPG are accelerated by  a momentum term. However,  as mentioned in \citep{li2015accelerated}, traditional Nesterov's accelerated method may encounter a bad momentum term for the non-convex optimization. To address the bad momentum term,  \cite{li2015accelerated} added another proximal operator as a monitor to make the objective function sufficient descent. To make the  objective functions generated from our AIPG  strictly descent, we follow the framework of APG  in \citep{li2015accelerated}. Thus, we compute two inexact proximal operators $z_{k+1} \in \textrm{Prox}^{\varepsilon_{k} }_{\gamma h}  \left ( y_{k} - \gamma  \nabla g(y_{k}) \right )$ and $v_{k+1} \in \textrm{Prox}^{\varepsilon_{k} }_{\gamma h} \left (x_{k} - \gamma \nabla g(x_{k})\right )$, where $v_{k+1}$  is a monitor to make the objective function strictly descent.  Specifically, our AIPG is presented in Algorithm \ref{algorithm2}.

\begin{algorithm}
\renewcommand{\algorithmicrequire}{\textbf{Input:}}
\renewcommand{\algorithmicensure}{\textbf{Output:}}
\caption{Accelerated inexact proximal gradient method (AIPG)}
\begin{algorithmic}[1]
\REQUIRE  $m$,  error $\varepsilon_k$ ($k=1,\cdots,m$), $t_0 = 0$, $t_1 = 1$, stepsize $\gamma < \frac{1}{L}$.
\ENSURE $x_{m+1}$.
 \STATE  Initialize  $x_0 \in \mathbb{R}^d$, and $x_1= z_1 = x_0$.
\FOR{$k=1,2,\cdots,m$}

\STATE  $y_k = x_k + \frac{t_{k-1}}{t_k} (z_k - x_k ) + \frac{t_{k-1 } -1 }{ t_k} (x_k - x_{k-1})$.
\STATE  Compute $z_{k+1}$ such that $z_{k+1} \in \textrm{Prox}^{\varepsilon_{k} }_{\gamma h}  \left ( y_{k} - \gamma  \nabla g(y_{k}) \right )$.
\STATE  Compute $v_{k+1} $ such that $v_{k+1} \in \textrm{Prox}^{\varepsilon_{k} }_{\gamma h} \left (x_{k} - \gamma \nabla g(x_{k})\right )$.

\STATE $t_{k+1} = \frac{\sqrt{4t_k^2 + 1} + 1}{2}$.
\STATE $x_{k+1} =  \left \{   \begin{array} {l@{\ \
\  \ }l} z_{k+1} & \textrm{if} \ \ f(z_{k+1}) \leq f(v_{k+1})
\\ v_{k+1} & \textrm{otherwise }  \end{array} \right . $

\ENDFOR
\end{algorithmic}
\label{algorithm2}
\end{algorithm}

To address the bad momentum term in the non-convex setting,  our AIPG (\emph{i.e.}, Algorithm \ref{algorithm2}) uses a pair of  inexact proximal operators to make the  objective functions strictly descent. Thus, AIPG is  a monotone  algorithm. Actually, using two proximal operators is a conservative strategy. As mentioned in \citep{li2015accelerated}, we can accept $z_{k+1}$ as $x_{k+1}$ directly if it satisfies the criterion $f(z_{k+1}) \leq f(x_k)  - \frac{\delta}{2} \| z_{k+1} -  y_k \|^2$ which shows that $y_k $  is a good extrapolation.  $v_{k+1}$ is computed only when this criterion is not met. Thus, the average number of  proximal operators can be reduced. Following this idea, we  propose our nonmonotone accelerated inexact proximal gradient algorithm (nmAIPG)  in Algorithm \ref{algorithm3}. Empirically, we find that the nmAIPG with the value of $\delta \in [0.5,1]$ works good. In our experiments, we  set $\delta=0.6$. %The average number of inexact proximal operators in nmAIPG can be reduced.
\begin{algorithm}
\renewcommand{\algorithmicrequire}{\textbf{Input:}}
\renewcommand{\algorithmicensure}{\textbf{Output:}}
\caption{Nonmonotone accelerated inexact proximal gradient method (nmAIPG)}
\begin{algorithmic}[1]
\REQUIRE  $m$, $\varepsilon_k$ ($k=1,\cdots,m$), $t_0 = 0$, $t_1 = 1$, stepsize $\gamma < \frac{1}{L}$, $\delta >0$.
\ENSURE $x_{m+1}$.
 \STATE  Initialize  $x_0 \in \mathbb{R}^d$, and $x_1= z_1 = x_0$.
\FOR{$k=1,2,\cdots,m$}

\STATE  $y_k = x_k + \frac{t_{k-1}}{t_k} (z_k - x_k ) + \frac{t_{k-1 } -1 }{ t_k} (x_k - x_{k-1})$.
\STATE  Compute $z_{k+1}$ such that $z_{k+1} \in \textrm{Prox}^{\varepsilon_{k} }_{\gamma h}  \left ( y_{k} - \gamma  \nabla g(y_{k}) \right )$.
\IF{$f(z_{k+1}) \leq f(x_k)  - \frac{\delta}{2} \| z_{k+1} -  y_k \|^2$}
\STATE $x_{k+1}  = z_{k+1}$
\ELSE
\STATE  Compute $v_{k+1} $ such that $v_{k+1} \in \textrm{Prox}^{\varepsilon_{k} }_{\gamma h} \left (x_{k} - \gamma \nabla g(x_{k})\right )$.
\STATE $x_{k+1} =  \left \{   \begin{array} {l@{\ \
\  \ }l} z_{k+1} & \textrm{if} \ \ f(z_{k+1}) \leq f(v_{k+1})
\\ v_{k+1} & \textrm{otherwise }  \end{array} \right . $
\ENDIF
\STATE $t_{k+1} = \frac{\sqrt{4t_k^2 + 1} + 1}{2}$.

\ENDFOR
\end{algorithmic}
\label{algorithm3}
\end{algorithm}
%%\vspace{-0.2cm}
\section{Convergence Analysis} \label{convergence_analysis}
As mentioned before,  the  convergence analysis of inexact proximal gradient methods for the  non-convex  problems   is still an open problem. This section will  address this challenge.

Specifically,  we first prove that IPG and AIPG converge to a critical point in the convex or non-convex setting (Theorem \ref{theorem3})  if $\{ \varepsilon_k\}$ is a decreasing sequence and $ \sum_{k=1}^m  \varepsilon_k < \infty $. Next, we  prove that IPG has the convergence rate $O(\frac{1}{T})$ for the non-convex problems (Theorem \ref{theorem2}) when the errors decrease at an appropriate rate.   Then, we  prove  the convergence rates for  AIPG in the non-convex setting (Theorem \ref{theorem4}). The detailed  proofs of  Theorems \ref{theorem3}, \ref{theorem2} and \ref{theorem4}   can be found in Appendix. % In addition, because our AIPG is different to the one in \citep{schmidt2011convergence}, we also prove
%the convergence rate $O(\frac{1}{T^2})$ for the convex optimization (Theorem \ref{theorem5}) when the errors decrease at an appropriate rate.

\subsection{Convergence  of IPG and AIPG}
We first   prove that IPGA and AIPG converge to a critical point (Theorem \ref{theorem3})  if $\{ \varepsilon_k\}$ is a decreasing sequence and $ \sum_{k=1}^m  \varepsilon_k < \infty $.
\begin{theorem} \label{theorem3} With Assumption \ref{NormalLipschitzconstant},
 if $\{ \varepsilon_k\}$ is a decreasing sequence and  $ \sum_{k=1}^m  \varepsilon_k < \infty $, we have $\textbf{0} \in \lim_{k\rightarrow \infty} \nabla g(x_k) + \partial_{\varepsilon_k} h({x}_k)$ for IPG and AIPG in the convex and non-convex optimization.
\end{theorem}
\begin{remark}\label{convergence_IPG_remark}
Theorem \ref{theorem3} guarantees that  IPG and AIPG converge to a critical point (or called as stationary point)  after an infinite number of iterations in the  convex or non-convex setting.
\end{remark}
%\vspace{-0.4cm}
\subsection{Convergence Rates of IPG}
Because  both the functions $g(x)$ and $h(x)$ are possibly non-convex, we cannot directly use $f(x_k) - f(x^*)$ or $\|x_k - x^* \|$ for analyzing the convergence rate of IPG, where $x^*$ is an optimal solution of (\ref{formulation1}).  In this paper, we use $\frac{1}{m}\sum_{k=1}^m \left \| x_k - x_{k-1} \right \|^2$  for analyzing the convergence rate of IPG in the non-convex setting. The detailed reason is provided in Appendix.  Theorem \ref{theorem2} shows  that IPG has the convergence rate $O(\frac{1}{T})$ for the non-convex optimization  when the errors decrease at an appropriate rate, which is exactly the same as  the error-free case (see  discussion in Remark \ref{convergence_IPG_remark}).

\begin{theorem} \label{theorem2}
For $g(x)$ is non-convex, and $h(x)$ is convex or non-convex, we have the  following results for  IPG:
\begin{enumerate}
\item If  $h(x)$ is convex,  we have that
\begin{eqnarray}\label{equ_theorem1}
&& \frac{1}{m}\sum_{k=1}^m \left \| x_k - x_{k-1} \right \|^2 \leq
\\ && \nonumber \frac{1}{m}\left ( 2A_m +  \sqrt{ \frac{1}{\frac{1}{\gamma} - \frac{L}{2}} \left (  f(x_{0}) - f(x^*)  \right )} + \sqrt{B_m}  \right )^2
\end{eqnarray}
where $A_m = { \frac{1}{2}\sum_{k=1}^m {\frac{1}{\frac{1}{\gamma} - \frac{L}{2}} \sqrt{\frac{2  \varepsilon_k}{\gamma} }} }$ and $B_m= {\frac{1}{\frac{1}{\gamma} - \frac{L}{2}} \sum_{k=1}^m  \varepsilon_k }$.
\item If  $h(x)$ is non-convex,  we have that
\begin{eqnarray}\label{equ_theorem2}
&& \frac{1}{m}\sum_{k=1}^m \left \| x_k - x_{k-1} \right \|^2
\\ &\leq& \nonumber {\frac{1}{ m \left (\frac{1}{2\gamma} - \frac{L}{2} \right )}  } \left ( f(x_{0}) -  f(x^*)  +  \sum_{k=1}^m  \varepsilon_k \right )
\end{eqnarray}
\end{enumerate}
\end{theorem}
\begin{remark}\label{convergence_IPG_remark}
Theorem \ref{theorem2} implies that IPG has the convergence rate $O(\frac{1}{T})$ for the non-convex optimization without errors. If  $\{ \sqrt{\varepsilon_k} \}$ is summable and $h(x)$ is  convex, we can also have that IPG has the convergence rate $O(\frac{1}{T})$ for the non-convex optimization. If  $\{ \varepsilon_k \}$ is summable and $h(x)$ is  non-convex, we can also have that IPG has the convergence rate $O(\frac{1}{T})$ for the non-convex optimization.
\end{remark}
%\vspace{-0.2cm}
\subsection{Convergence Rates of AIPG }
In this section, based on  the $\varepsilon$-KL property, we  prove that  AIPG  converges in a finite number of iterations, in a  linear rate  or a sublinear rate  at different conditions in the non-convex setting (Theorem \ref{theorem4}), which is exactly the same as  the error-free case \citep{li2015accelerated}.

%\subsubsection{Nonconvex optimization}
% Based on Lemma \ref{lemma1}, we prove the convergence rate of AIPG for non-convex optimization (Theorem \ref{theorem4}).
%\begin{lemma} \label{lemma1}
%Let $\Omega$ be a compact set and let $f(x) : \mathbb{R}^N \rightarrow (-\infty, +\infty ]$  be a proper and lower semicontinuous function. Assume that $f(x)$ is constant on $\Omega$ and satisfies the $\varepsilon$-KL property at each point of $\Omega$. Then there exists $\epsilon>0$, $\eta>0$ and $\varphi \in \Phi_{\eta}$, such that for all $\overline{u} \in \Omega$ and all $u$ in the following intersection
%\begin{eqnarray}\label{lemma_equ1}
%&& \{u \in \mathbb{R}^N: dist(u,\Omega) <\epsilon \}
%\\ && \cap \{u \in \mathbb{R}^N:  f(\overline{u}) < f({u}) <  f(\overline{u})  + \eta \} \nonumber
%\end{eqnarray}
%the following inequality holds
%\begin{eqnarray} \label{lemma_equ2}
%\varphi'(f(u) - f(\overline{u})) dist (\textbf{0}, \nabla g(u) + \partial_{\varepsilon} h(u))) \geq 1
%\end{eqnarray}
%\end{lemma}

\begin{theorem} \label{theorem4}
Assume that $g$ is a non-convex function with Lipschitz continuous gradients, $h$ is a proper and lower semicontinuous function. If  the function $f$ satisfies the $\varepsilon$-KL property, $\varepsilon_k = \alpha \left \| v_{k+1} - x_{k} \right \|^2$, $\alpha \geq 0$, $ \frac{1}{2 \gamma} - \frac{L}{2} -  \alpha \geq 0$ and the desingularising function has the form $\varphi(t)=\frac{C}{\theta} t^\theta$ for some $C>0$, $\theta \in (0,1]$, then
\begin{enumerate}
\item If $\theta=1$, there exists $k_1$ such that $f(x_k)= f^* $ for all $k>k_1$ and AIPG terminates in a finite number of steps, where $\lim_{k \rightarrow \infty } f(x_k) = f^*$.
\item If $\theta \in [\frac{1}{2},1)$, there exists $k_2$ such that  for all $k>k_2$
\begin{eqnarray} \label{thm4_1.1}
 f(x_{k}) - \lim_{k\rightarrow \infty} f(x_k)
 \leq  \left ( \frac{d_1C^2}{1+d_1 C^2} \right )^{k-k_2} \left (f(v_{k}) - f^* \right )  \nonumber
\end{eqnarray}
where $d_1 =  \frac{\left  ( \frac{1}{\gamma } + L  + \sqrt{ \frac{2 \alpha}{\gamma} } \right )^2}{  \frac{1}{2 \gamma} - \frac{L}{2} -  \alpha  }$.
\item If $\theta \in (0,\frac{1}{2})$, there exists $k_3$ such that  for all $k>k_3$
\begin{eqnarray} \label{thm4_1.1}
 f(x_{k}) - \lim_{k\rightarrow \infty} f(x_k)
\leq  \left ( \frac{C}{(k-k_3)d_2 (1-2 \theta)} \right )^{\frac{1}{1-2 \theta}} \nonumber
\end{eqnarray}
where
\begin{eqnarray}
\nonumber d_2 =  \min \left  \{ \frac{1}{2d_1 C}, \frac{C}{1-2\theta} \left ( 2^{\frac{2 \theta -1}{2 \theta -2}} -1 \right ) \left (f(v_{0}) - f^* \right )^{2 \theta -1}\right  \}
\end{eqnarray}
\end{enumerate}
\end{theorem}
\begin{remark}\label{convergence_AIPG_remark}
Theorem \ref{theorem4} implies that AIPG converges in a finite number of iterations when $\theta=1$, in a  linear rate when  $\theta \in [\frac{1}{2},1)$ and at least a sublinear rate when $\theta \in (0,\frac{1}{2})$.
\end{remark}
\section{Experiments}\label{experiments}
We first give the experimental setup, then
present the implementations of three non-convex applications with the corresponding experimental results.
%\vspace{-4pt}
\subsection{Experimental Setup}
\subsubsection{Design of Experiments}
To validate the effectiveness  of our inexact proximal gradients methods (\emph{i.e.},  IPG, AIPG and
nmAIPG), we apply them to solve three  representative non-convex learning problems  as follows.
\begin{enumerate}[leftmargin=0.2in]
\item \textbf{Robust OSCAR:} Robust OSCAR is a robust version of OSCAR method \citep{zhong2012efficient}, which is a feature selection model with the capability to automatically detect feature group structure.  The function $g(x)$ is non-convex, and the   function $h(x)$ is convex. {Zhong and Kwok} \citep{zhong2012efficient} proposed a fast iterative group merging  algorithm for exactly solving the proximal operator. % Our inexact proximal gradient algorithms are \textbf{significantly faster } than the exact proximal gradient algorithms.
\item \textbf{Social Link Prediction:} Given  an incomplete matrix $M$ (user-by-user) with each entry $M_{ij}\in \{ +1,-1 \}$, social link prediction  is to recover the matrix $M$ (\emph{i.e.} predict the potential social links or friendships between users) with low-rank constraint. The function $g(x)$ is convex, and the   function $h(x)$ is non-convex. The low-rank proximal operator can be solved exactly by the Lanczos method \citep{larsen1998lanczos}. % Our inexact proximal gradient algorithms are  \textbf{faster} than the exact proximal gradient algorithms.
\item \textbf{Robust Trace Lasso:} Robust trace Lasso  is a robust version of trace Lasso \citep{grave2011trace}. The   function $g(x)$ is non-convex, and the   function $h(x)$ is convex. To the best of our knowledge, there  is still no   proximal gradient algorithm  for solving trace Lasso. % Our inexact proximal gradient algorithms \textbf{fill this vacancy}.
\end{enumerate}
We also summarize these three non-convex learning problems in Table \ref{table:Application}. To show the advantages of our  inexact proximal gradients methods, we compare the convergence rates and the running time for our  inexact proximal gradients methods and the exact proximal gradients methods (\emph{i.e.},  PG, APG and
nmAPG).
 \begin{table}[htbp]
%\small
 \center
 \caption{Three typical  learning applications. (C, NC and PO are the  abbreviations of convex, non-convex and proximal operator, respectively.)}
%\vspace{-0.3cm} \setlength{\tabcolsep}{3mm}
\begin{tabular}{c|c|c|c}
\hline
\textbf{Application}  &  $g(x)$ &    $h(x)$  &  \textbf{Exact PO}  \\ \hline
Robust OSCAR  &  NC  & C  & Yes  \\
Link Prediction  & C   &  NC & Yes  \\
Robust Trace Lasso  & NC   & C  & No
\\ \hline
\end{tabular}
\label{table:Application}
\end{table}

%\vspace{-0.4cm}
\subsubsection{Datasets}
Table \ref{table:datasets} summarizes the six  datasets used  in our experiments.
 Specifically, the  Cardiac and Coil20 datasets are for the robust OSCAR application. The Cardiac dataset was  collected by us from   hospital, which is to predict the area of
the left  ventricle, and is encouraged to find the homogenous groups of features. The Coil20 dataset is from \url{http://www1.cs.columbia.edu/CAVE/software/softlib/coil-20.php}.   The Soc-sign-epinions and Soc-Epinions1 datasets are the social network data for the social link prediction application, where each row corresponds to a node, and each collum corresponds to an edge. They are from \url{http://snap.stanford.edu/data}. The  GasSensorArrayDrift and YearPredictionMSD datasets are for the robust trace Lasso application. They are from the UCI benchmark repository  \url{https://archive.ics.uci.edu/ml/datasets.html}.

\begin{table}[htbp]
\scriptsize
%\vspace{-6pt}
\center \caption{The  datasets used in the experiments.}
%\vspace{-6pt}
\setlength{\tabcolsep}{0.7mm}
\begin{tabular}{c|c|c|c}
\hline \textbf{Application} & \textbf{Dataset} &   \textbf{Sample size}    & \textbf{Attributes} \\
 \hline
 \multirow{2}{1.5cm}{Robust OSCAR}  & Cardiac & 3,360   & 800  \\
 & Coil20 & 1,440  & 1,024  \\  \hline
 \multirow{2}{1.5cm}{Social Link Prediction} & Soc-sign-epinions(SSE) & 131,828  & 841,372  \\
 & Soc-Epinions1(SE) & 75,879 & 508,837  \\  \hline
 \multirow{2}{1.5cm}{Robust Trace Lasso}  & GasSensorArrayDrift(GS) & 13,910 & 128 \\
 & YearPredictionMSD(YP)  & 51,5345 & 90 \\
\hline
\end{tabular}
\label{table:datasets}
\end{table}
\begin{figure*}[th]
        %%\vspace{-0.5in}
       %\vspace*{-8pt}
	\centering
	\begin{subfigure}[b]{0.245\textwidth}
		\centering
		\includegraphics[width=1.75in]{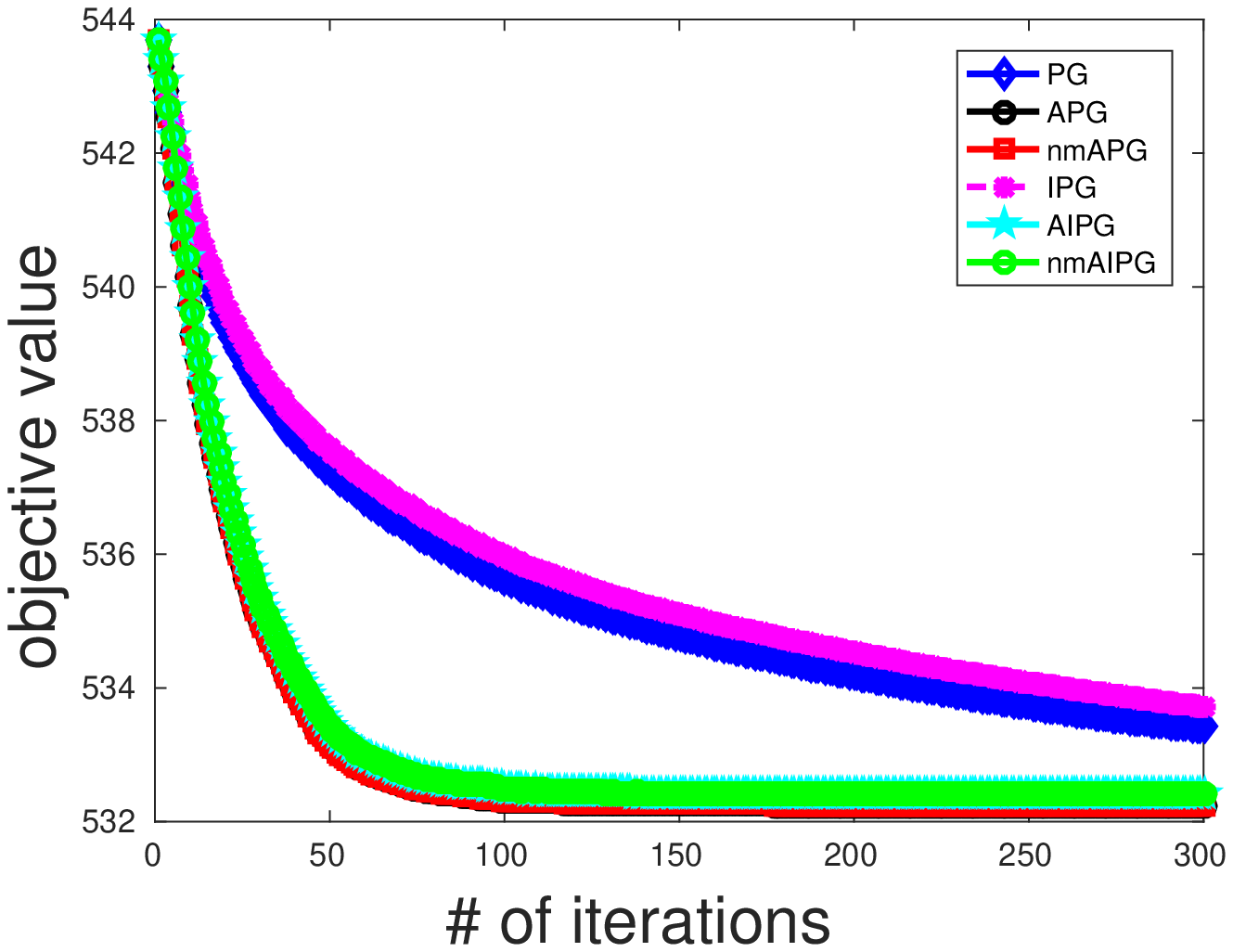}
		\caption{Cardiac: obj. vs. iteration}
\label{OSCAR_MNIST_iteration}
	\end{subfigure}
	\begin{subfigure}[b]{0.245\textwidth}
		\centering
		\includegraphics[width=1.75in]{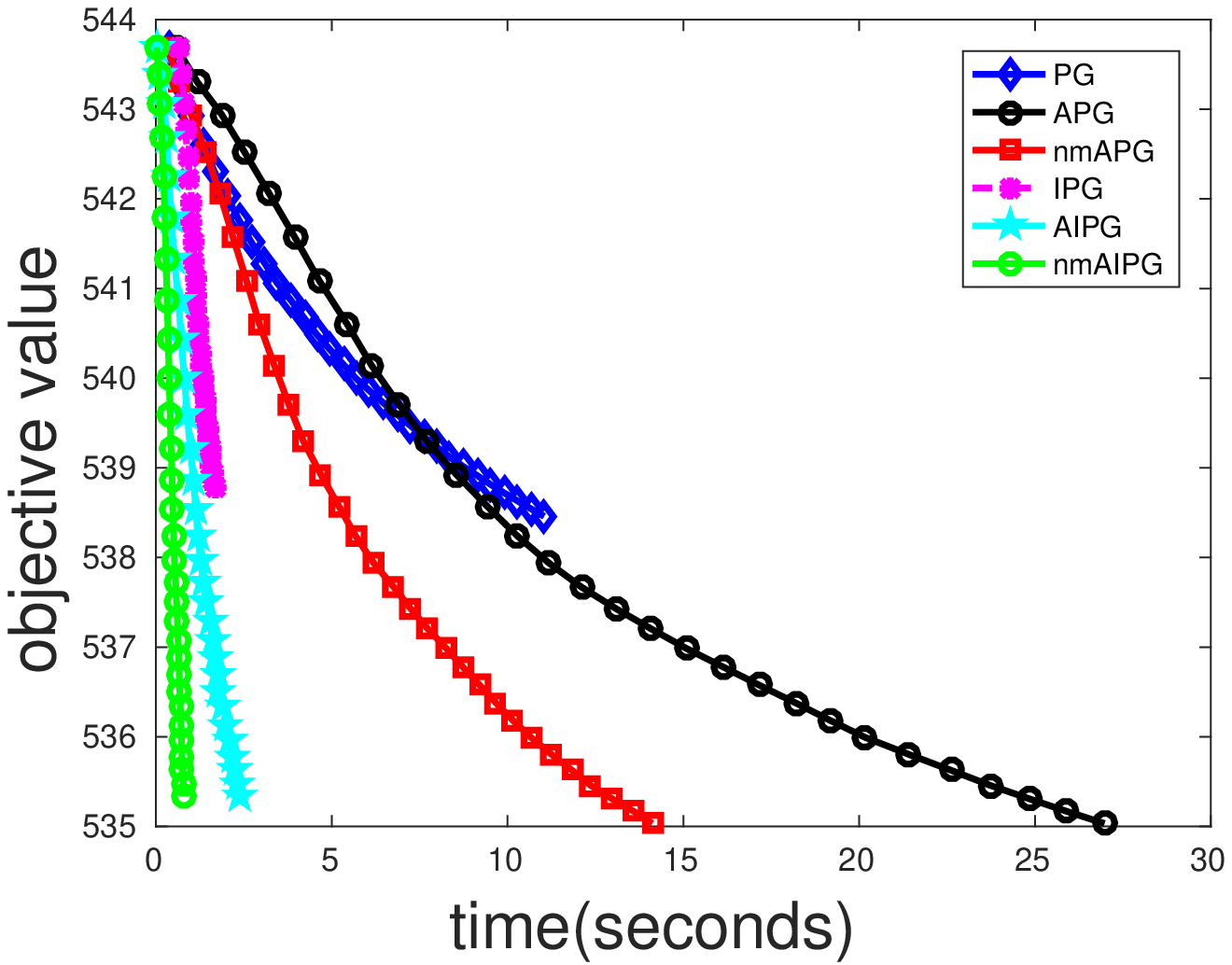}
		\caption{Cardiac: obj. vs. time}
\label{OSCAR_MNIST_time}
	\end{subfigure}
	\begin{subfigure}[b]{0.245\textwidth}
		\centering
		\includegraphics[width=1.75in]{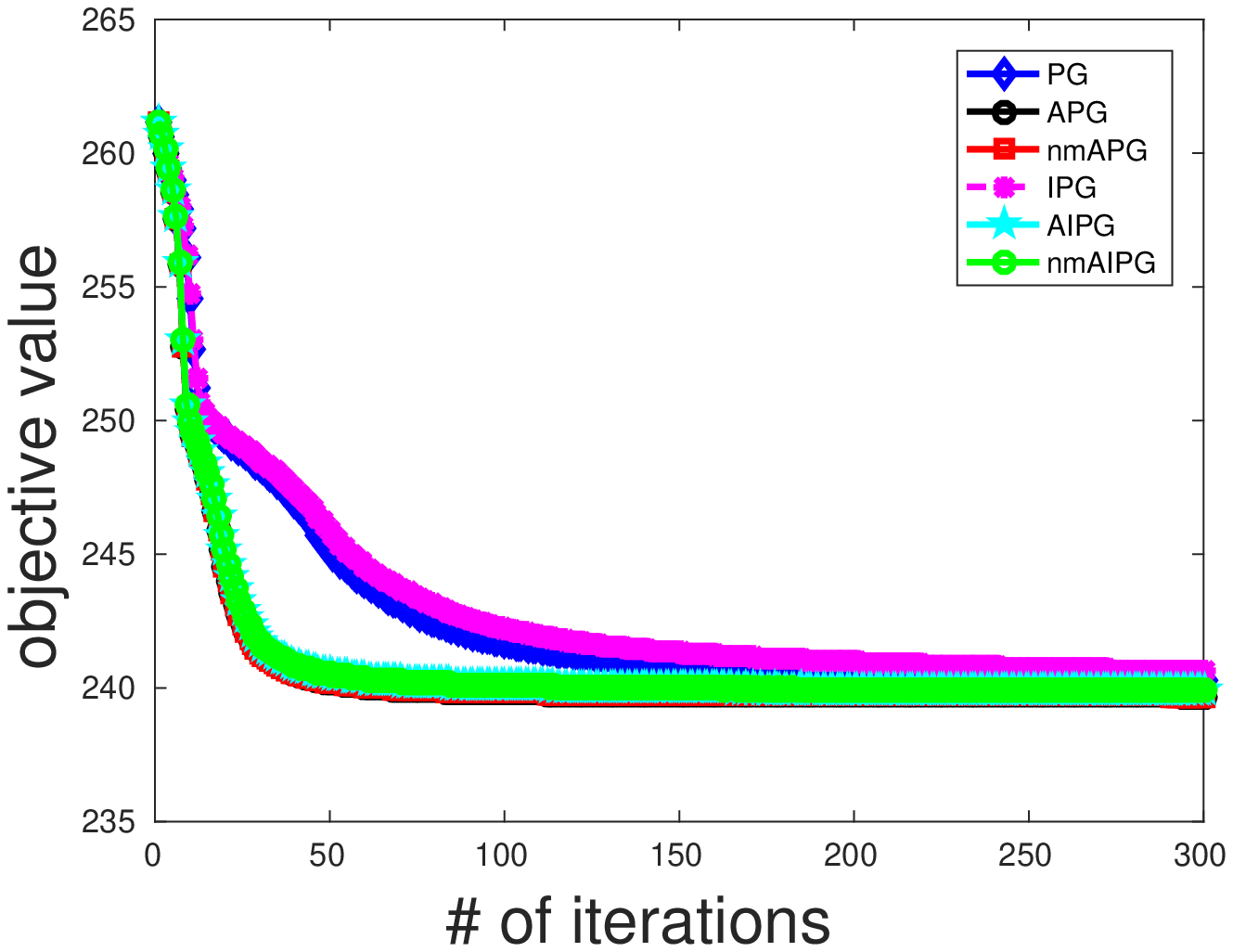}
		\caption{Coil20: obj. vs. iteration}
\label{OSCAR_coil20_iteration}
	\end{subfigure}
	\begin{subfigure}[b]{0.245\textwidth}
		\centering
		\includegraphics[width=1.75in]{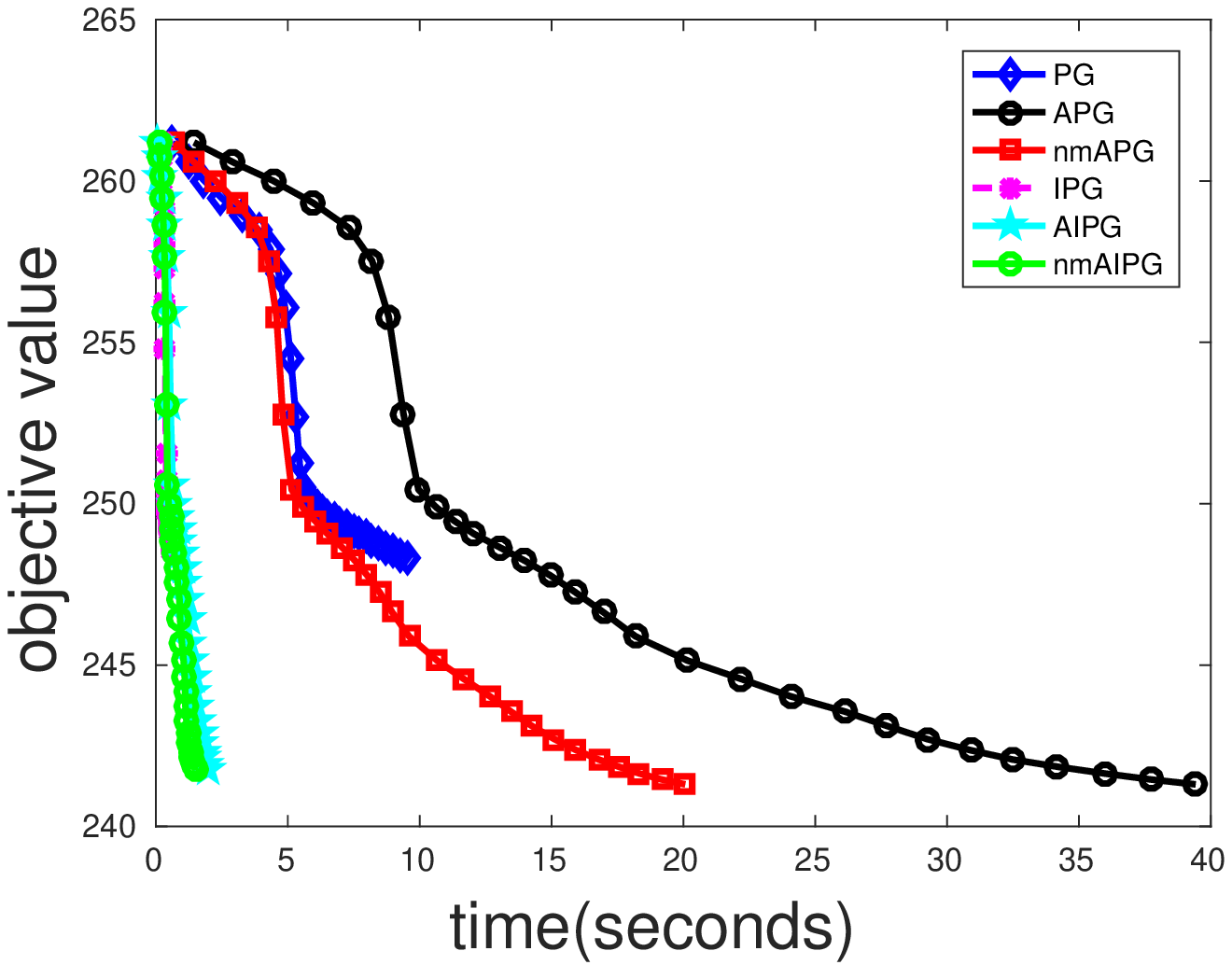}
		\caption{Coil20: obj. vs time}
\label{OSCAR_coil20_time}
	\end{subfigure}
%%\vspace*{-8pt}
%	\centering

	\begin{subfigure}[b]{0.245\textwidth}
		\centering
		\includegraphics[width=1.75in]{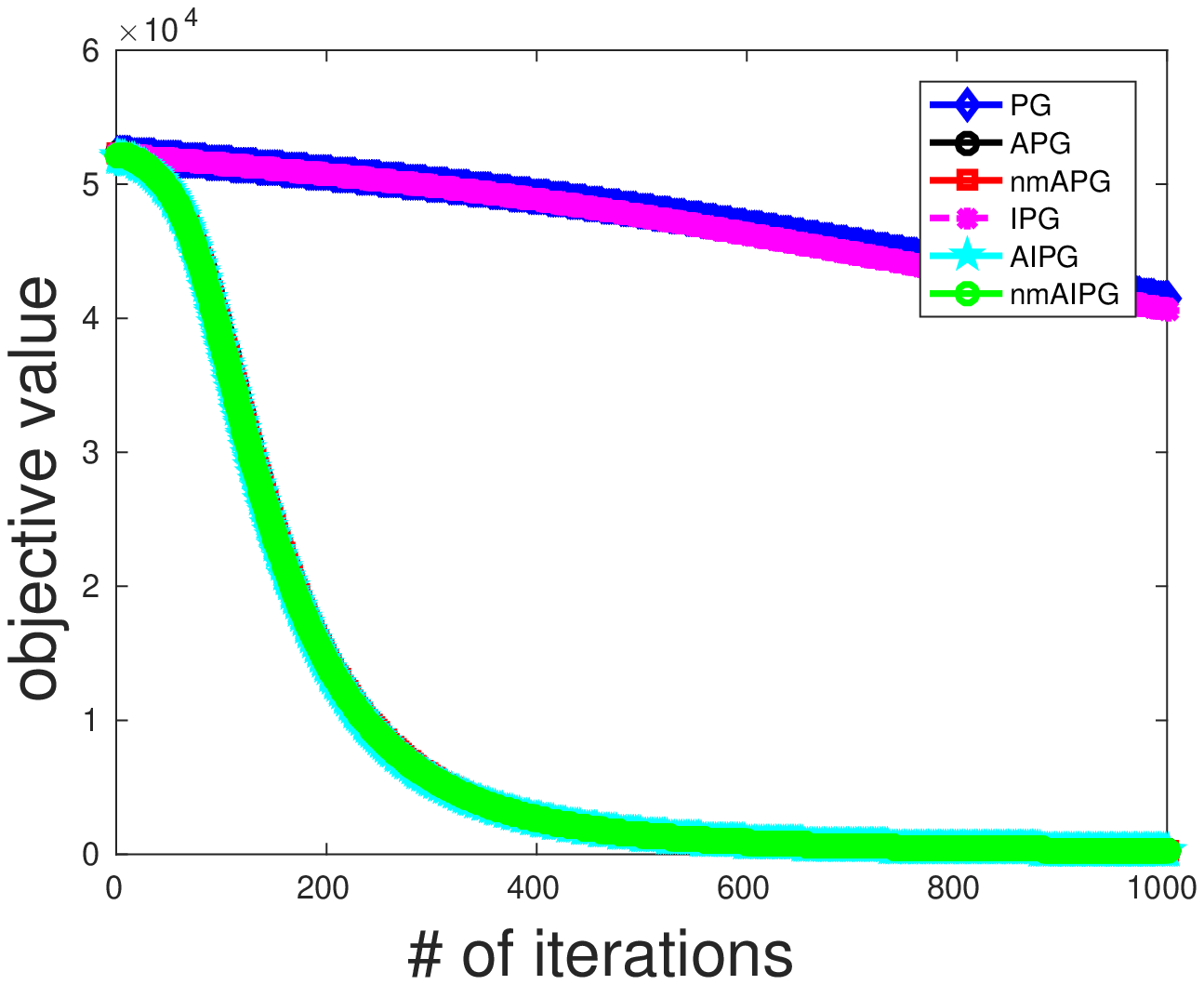}
		\caption{SSE: obj. vs. iteration}
\label{link_Epinions_iteration}
	\end{subfigure}
	\begin{subfigure}[b]{0.245\textwidth}
		\centering
		\includegraphics[width=1.75in]{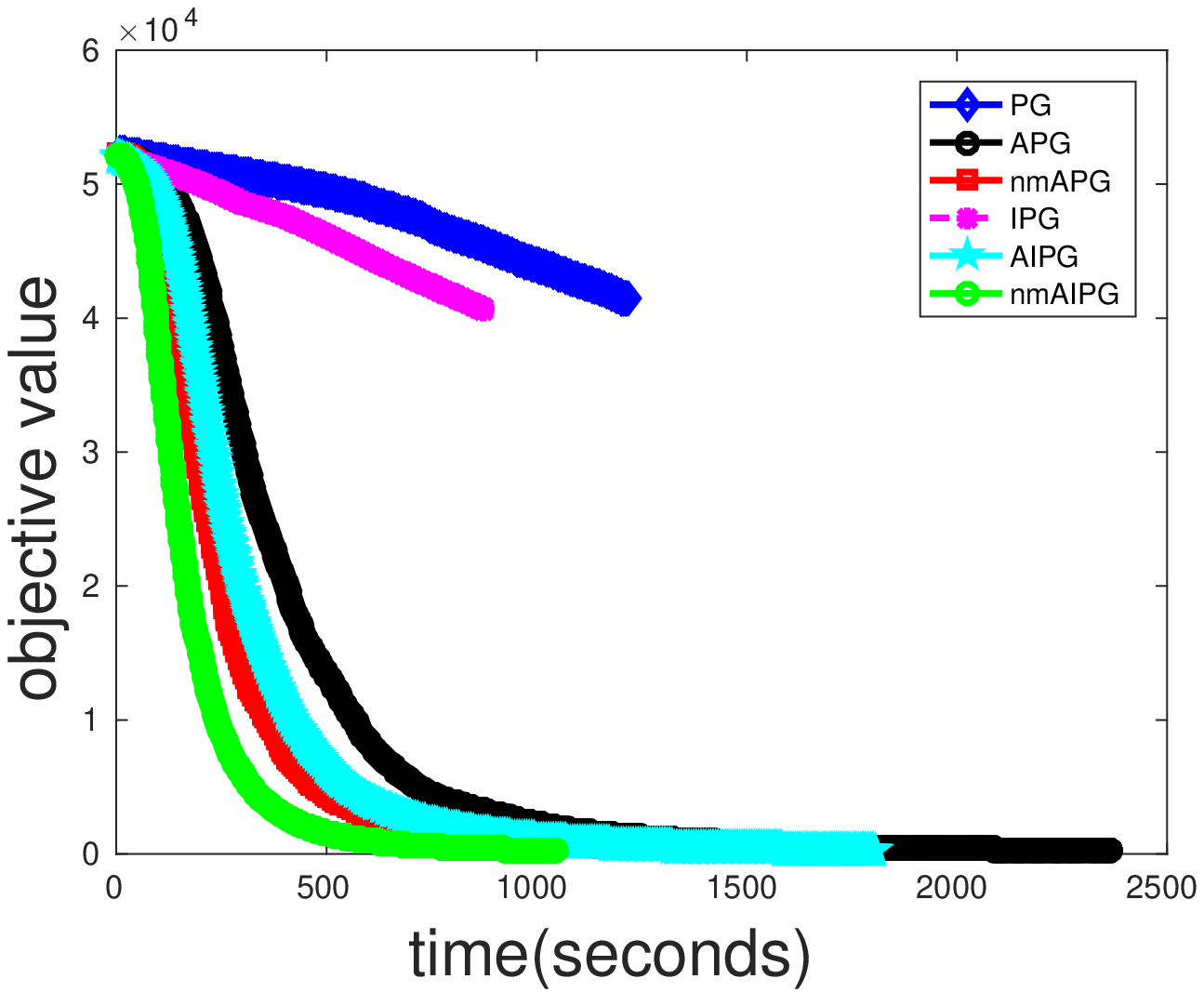}
		\caption{SSE: obj. vs. time}
\label{link_Epinions_time}
	\end{subfigure}
	\begin{subfigure}[b]{0.245\textwidth}
		\centering
		\includegraphics[width=1.8in]{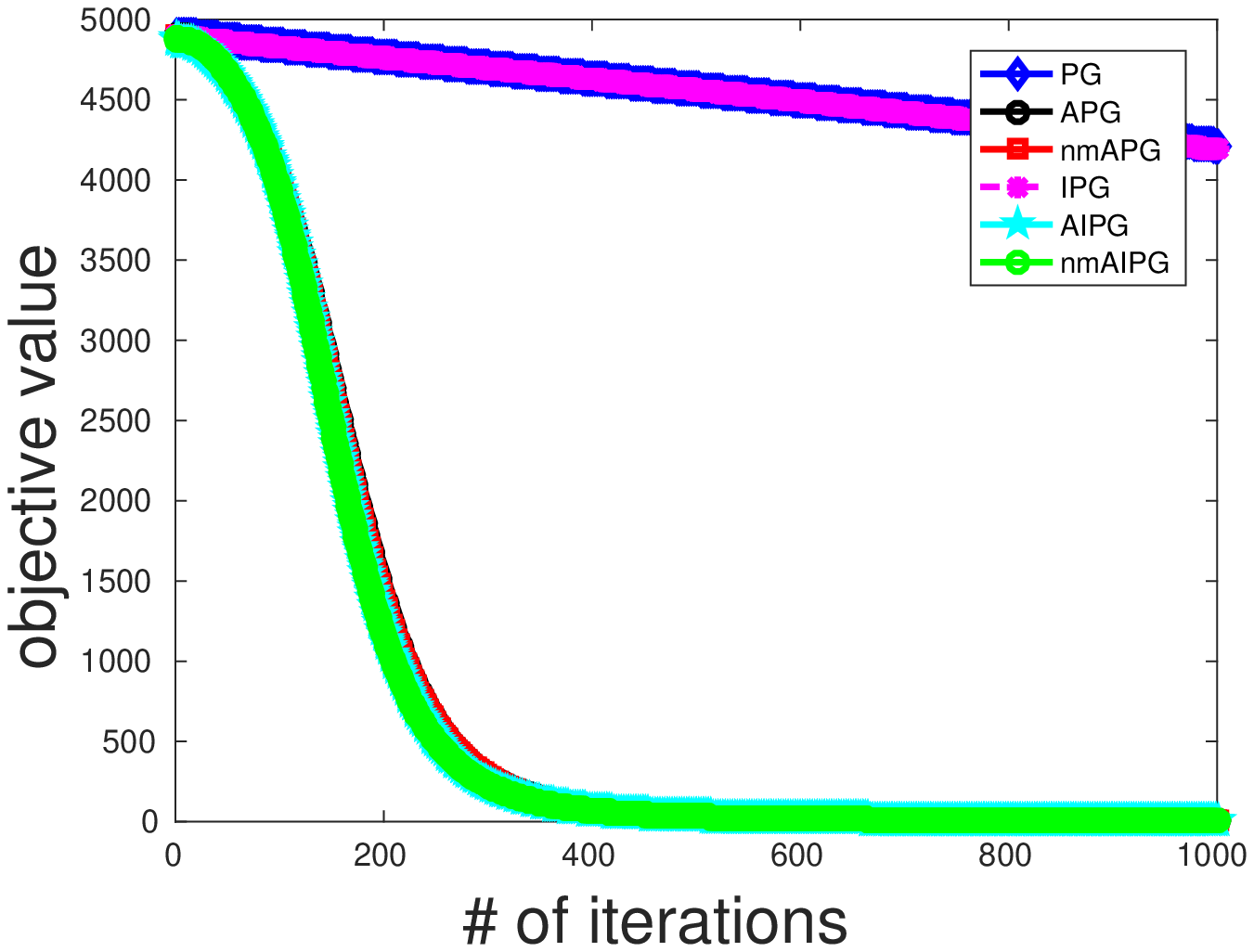}
		\caption{SE: obj. vs. iteration}
\label{link_trust_iteration}
	\end{subfigure}
	\begin{subfigure}[b]{0.245\textwidth}
		\centering
		\includegraphics[width=1.8in]{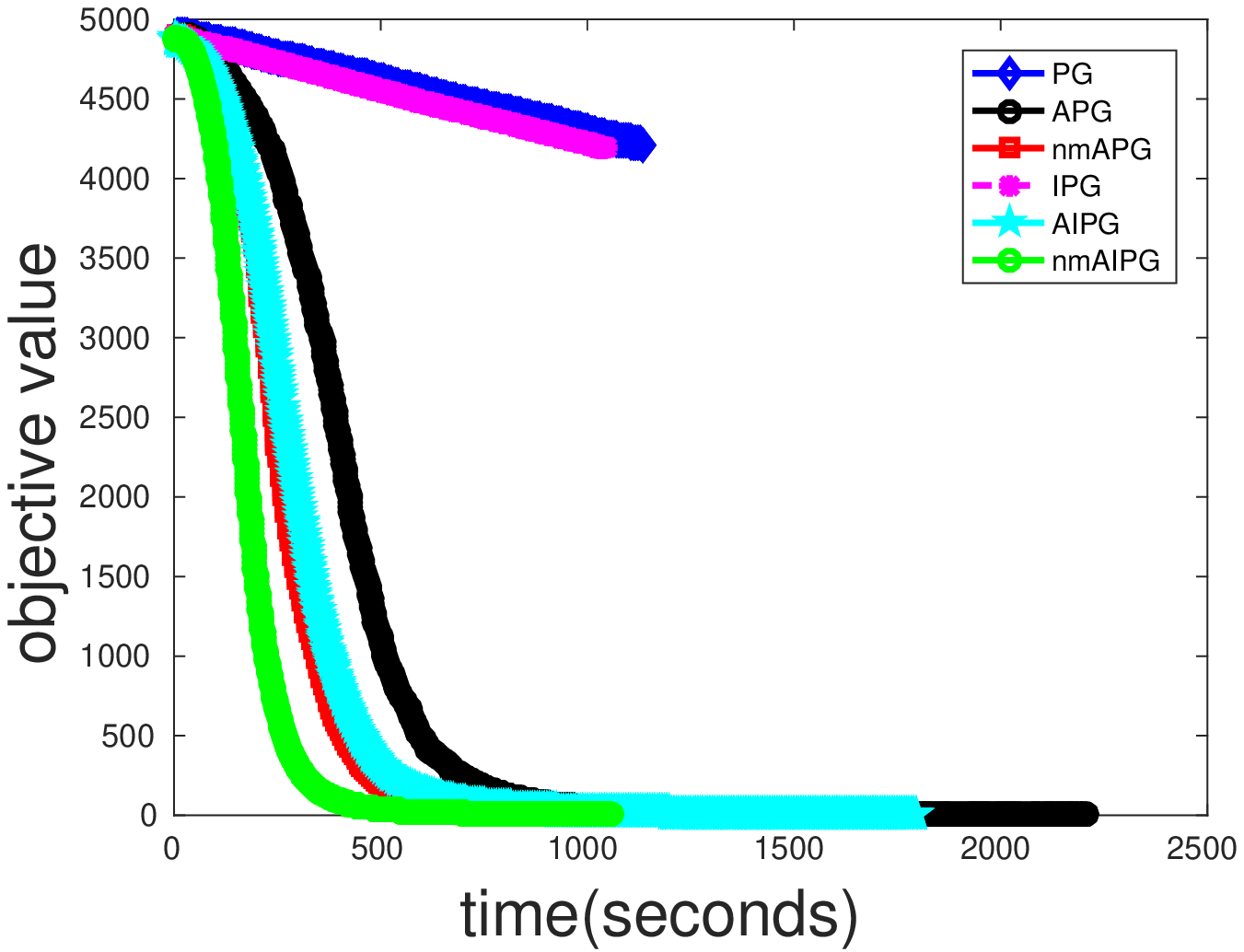}
		\caption{SE: obj. vs. time}
\label{link_trust_time}
	\end{subfigure}

	\begin{subfigure}[b]{0.245\textwidth}
		\centering
		\includegraphics[width=1.75in]{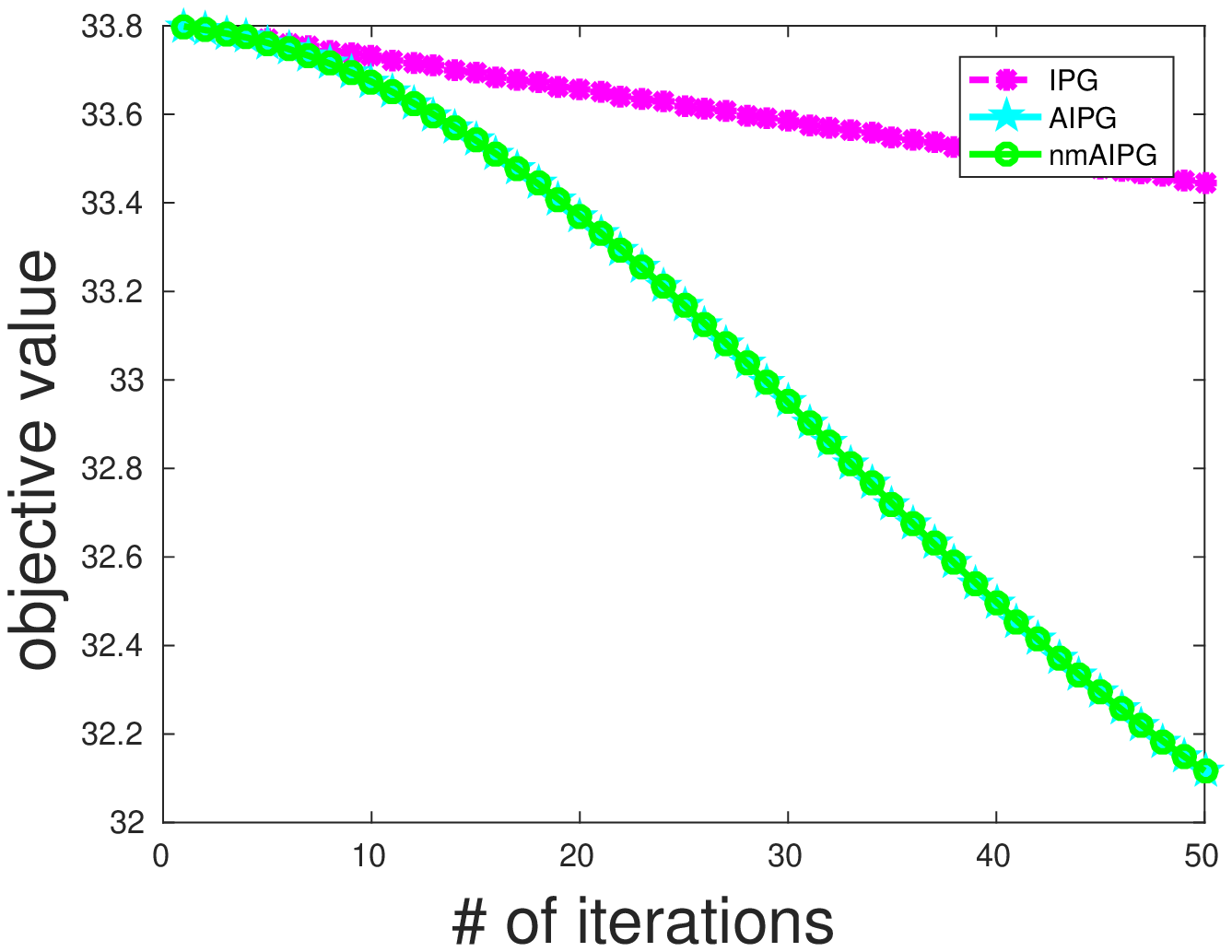}
		\caption{GS: obj. vs. iteration}
\label{GS_iteration}
	\end{subfigure}
	\begin{subfigure}[b]{0.245\textwidth}
		\centering
		\includegraphics[width=1.75in]{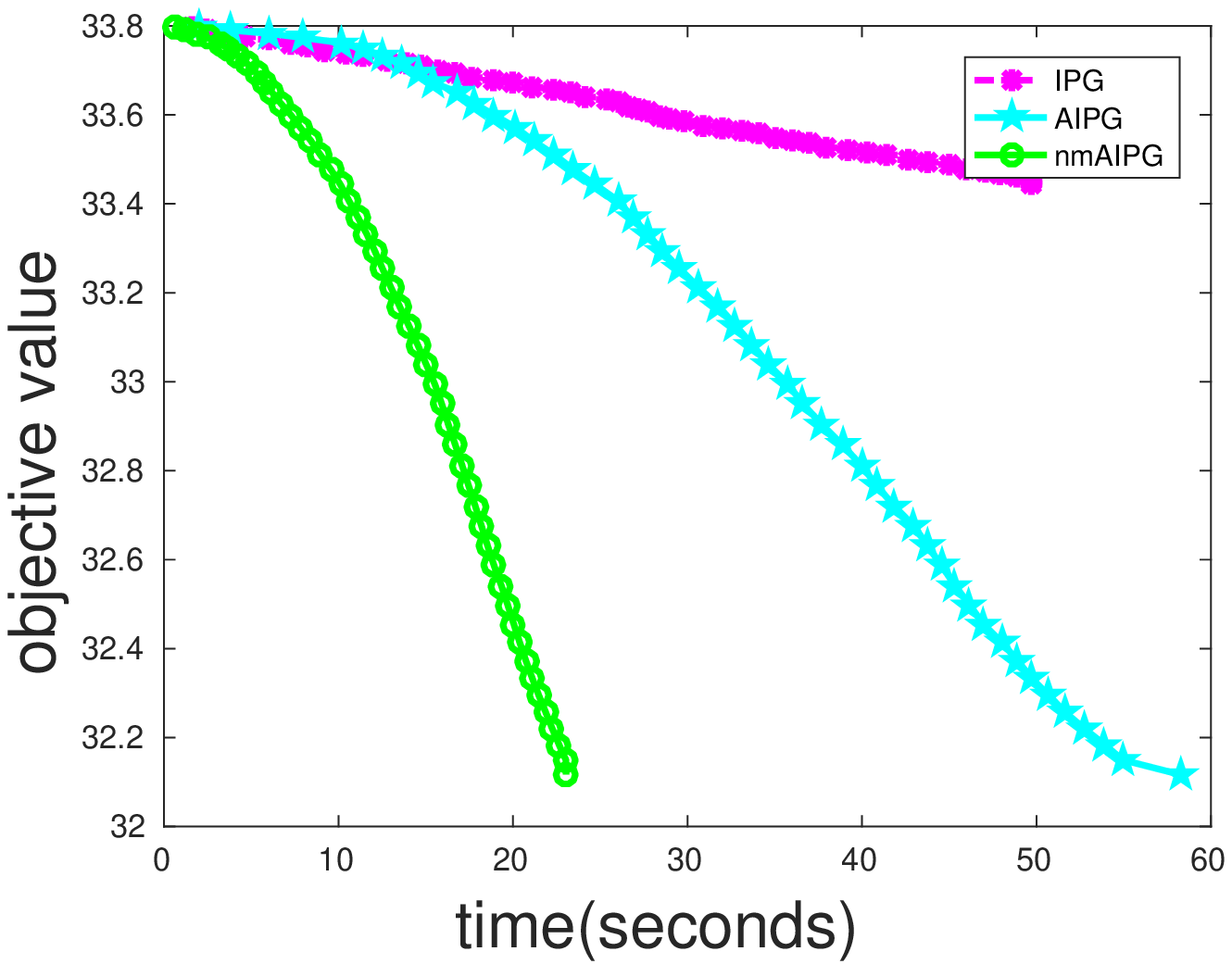}
		\caption{GS: obj. vs. time}
	\end{subfigure}
	\begin{subfigure}[b]{0.245\textwidth}
		\centering
		\includegraphics[width=1.75in]{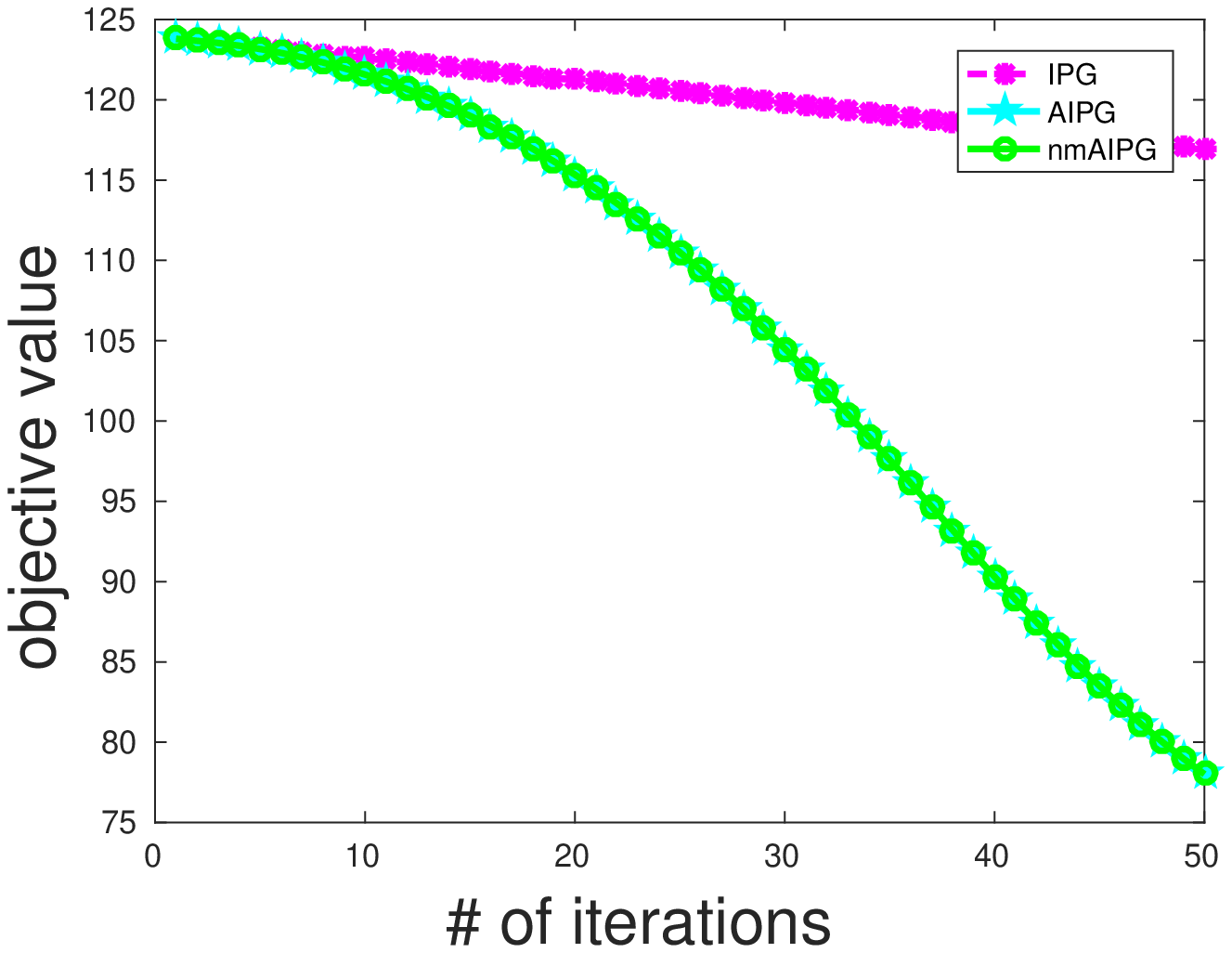}
		\caption{YP: obj. vs. iteration}
	\end{subfigure}
	\begin{subfigure}[b]{0.245\textwidth}
		\centering
		\includegraphics[width=1.75in]{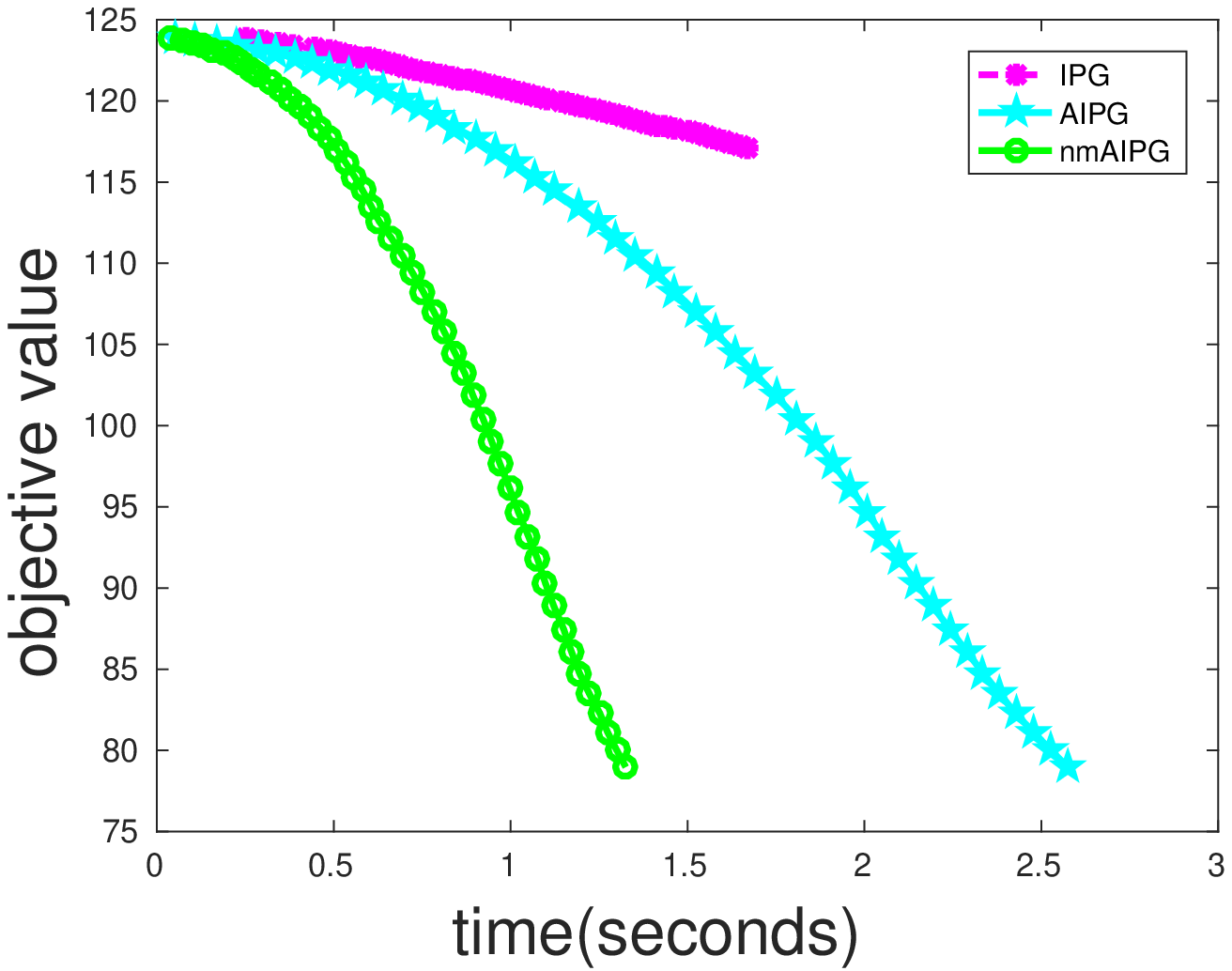}
		\caption{YP: obj. vs. time}
\label{YP_time}
	\end{subfigure}
\setlength{\abovecaptionskip}{-0.1pt}
%\vspace{-6pt}
\caption{Comparison of convergence speed  for different methods. (a-d):  Robust OSCAR. (e-h): Link prediction. (i-l) Robust trace LASSO.}
\label{link}
%\vspace{-14pt}
\end{figure*}
%\vspace{-10pt}
\subsection{ Implementations and  Exerimental Results}\label{implentation}
\subsubsection{Robust OSCAR}
For the  robust regression, we replace the square loss originally used in OSCAR  with the  correntropy induced loss \citep{he2011maximum}. Thus, we consider the robust OSCAR with the functions $g(x )$ and $h(x )$ as follows.
\begin{eqnarray}\label{formulation5.1.1}
g(x ) &=&
 \frac{\sigma^2}{2}\sum_{i=1}^{l}\left (1 -e^{ - \frac{ (y_i -  X_i^Tx)^2 }{\sigma^2 }}\right ) \,,
\\  h(x ) &=& \lambda_1  \| x \|_1 +\lambda_2 \sum_{i<j} \max \{ | x_i | , |x_j  | \}\,,
\end{eqnarray}
where  $\lambda_1\geq 0$ and $\lambda_2\geq 0$ are two  regularization parameters.
 For exact proximal gradients algorithms, {Zhong and Kwok} \cite{zhong2012efficient} proposed a fast iterative group merging  algorithm for exactly solving the proximal subproblem. We propose a subgradient algorithm to approximately solve the proximal subproblem. Let $o(j) \in \{1,2,\cdots,N \}$ denote the order  of $| x_j |$ \footnote{Here, $x_j$ denotes the $j$-th coordinate of the vector $x$.} among $\{| x_{1} |,| x_{2} |,\cdots,| x_{N} | \}$ such that if    $o(j_1)<o(j_2) $, we have
$|x_{j_1 } | \leq  | x_{j_2 } |$. Thus, the  subgradient of $h(x)$  can be computed as $\partial h(x) = \sum_{j=1}^{N}\left (\lambda_1 + \lambda_2 (o(j)-1)  \right  ) \partial | x_j |$.
The  subgradient algorithm  is omitted here due to the space limit.
We implement our IPG,  AIPG and nmAIPG methods for robust OSCAR  in Matlab. We also implement PG, APG and nmAPG in Matlab.

 Figures \ref{OSCAR_MNIST_iteration}  and \ref{OSCAR_coil20_iteration} show the convergence rates of the objective value  vs. iteration for the exact and inexact proximal gradients methods. The results confirm that exact and inexact proximal gradients methods have the
same convergence rates. The convergence rates of  exact and inexact accelerated methods  are faster than the ones of the basic methods (PG and IPG). Figures  \ref{OSCAR_MNIST_time}  and \ref{OSCAR_coil20_time} show the convergence rates of the objective value  vs. the running time for the exact and inexact proximal gradients methods. The  results show that our inexact methods are \textbf{significantly faster} than the exact  methods. When the dimension increases,  we can even achieve more than \textbf{100 folds speedup}. This is because our subgradient based  algorithm for  approximately solving the proximal subproblem  is much efficient than  the   projection algorithm for exactly solving the proximal subproblem \cite{zhong2012efficient}.

\subsubsection{Social Link Prediction}
In social link prediction problem, we hope to predict the new potential social links or friendships between online users. Given  an incomplete matrix $M$ (user-by-user) with each entry $M_{ij}\in \{ +1,-1 \}$, social link prediction  is to recover the matrix $M$ with low-rank constraint. Specifically, social link prediction considers the function $f(X)$  as following:
\begin{eqnarray}\label{formulation5.1.1}
\min_{X} && \underbrace{\frac{1}{2}\sum_{(i,j) \in \Omega}  \log (1+ \exp (-X_{ij} M_{ij}))}_{g(X)}
\\  \nonumber s.t.&&  \textrm{rank}(X) \leq r \,,
\end{eqnarray}
where  $\Omega$  is a set of $(i,j)$ corresponding to the entries of $M$ which are observed, $\lambda$ is a regularization parameter.
 The proximal operator  $ \min_{\textrm{rank}(X) \leq r}  \left \| X- \left (X_{t-1} - \gamma  \nabla g(X_{t-1}) \right ) \right \|^2
$ can be solved by the \textrm{rank}-$r$ singular value decomposition (SVD) \citep{jain2010guaranteed}. The  \textrm{rank}-$r$ SVD can be solved exactly by the  Lanczos method \citep{larsen1998lanczos}, and also can be solved approximately  by the power method  \cite{halko2011finding,journee2010generalized}.  We implement our IPG,  AIPG and nmAIPG methods for social link prediction  in Matlab. We also implement PG, APG and nmAPG in Matlab.

Figures \ref{link_Epinions_iteration}  and \ref{link_trust_iteration} show  the convergence rates of the objective value  vs. iteration for the exact and inexact proximal gradients methods. The results confirm that exact and inexact proximal gradients methods have the
same convergence rates.  Figures \ref{link_Epinions_time} and \ref{link_trust_time} illustrate the convergence rates of the
objective value vs.  running time for the exact and inexact proximal gradients methods. The  results verify that our inexact methods are \textbf{faster} than the exact  methods.

\subsubsection{Robust Trace Lasso}
Robust trace Lasso  is a robust version of trace Lasso \citep{grave2011trace,bach2008consistency}. Same with robust OSCAR, we  replace the square loss originally used in trace Lasso  with the  correntropy induced loss \citep{he2011maximum} for robust regression.
Thus, we consider the robust trace Lasso  with the functions $g(x )$ and $h(x )$ as following:
\begin{eqnarray}\label{formulation5.1.1}
g(x ) &=&
 \frac{\sigma^2}{2}\sum_{i=1}^{l}\left (1 -e^{ - \frac{ (y_i -  X_i^Tx)^2 }{\sigma^2 }}\right )
\\  h(x ) &=& \lambda \left \| X \textrm{Diag}(x) \right \|_* \,,
\end{eqnarray}
where $\left \| \cdot \right \|_*$  is the trace norm, $\lambda$ is a regularization parameter.
To the best of our knowledge, there is still no  algorithm to  exactly solve the proximal subproblem. To implement our IPG and AIPG,  we propose a subgradient algorithm  to approximately solve the proximal subproblem. Specifically, the subgradient of the trace Lasso regularization $\left \| X \textrm{Diag}(x) \right \|_* $ can be computed by Theorem \ref{thm_subg_trace} which is originally provided in \cite{bach2008consistency}. We implement our IPG,  AIPG and nmAIPG methods for  robust trace Lasso in Matlab.
\begin{theorem}\label{thm_subg_trace}
Let $U \textrm{Diag}(s)V^T$ be the singular value decomposition of $X \textrm{Diag}(x)$. Then, the subgradient of the trace Lasso regularization $\left \| X \textrm{Diag}(x) \right \|_* $ is given by
\begin{eqnarray}
&& \partial \left \| X \textrm{Diag}(x) \right \|_* = \{ \textrm{Diag}\left (X^T (UV^T+M) \right ) :
\\ \nonumber  && \quad  \quad  \quad  \quad     \|M\|_2 \leq 1, U^TM=0 \ \
  and \ \ MV=0  \}
\end{eqnarray}
\end{theorem}

 Figures \ref{GS_iteration}-\ref{YP_time} show the convergence rates of the
objective value vs. iteration and  running time respectively,  for our inexact proximal gradient methods. The results demonstrate that we provide  efficient proximal gradient algorithms for the robust trace Lasso. More importantly, our IPG, AIPG  and nmAIPG algorithms   \textbf{fill the vacancy} that there is no proximal gradient algorithm for trace Lasso. This is because that, directly solving the  proximal subproblem for robust trace Lasso is quite difficult \citep{grave2011trace,bach2008consistency}. Our subgradient based
algorithm  provides an alternative approach for  solving the proximal subproblem.
%\vspace{-4pt}
\subsubsection{Summary of the Experimental results}
Based on the results of three non-convex machine learning applications, our conclusion is that our  inexact proximal gradient algorithms can provide \textbf{flexible}  algorithms to the  optimization problems with
complex non-smooth regularization. More importantly, our  inexact   algorithms could  be  \textbf{significantly faster} than the  exact proximal gradient algorithms.
%\vspace{-4pt}
\section{Conclusion}\label{conclusion}
Existing inexact proximal  gradient methods only consider  convex problems. The knowledge of inexact proximal gradient methods in the non-convex setting is very limited. % Moreover, for some machine learning models, there is still no proposed solver for exactly solving the  proximal operator.
 To address this challenge, in this paper, we first propose three  inexact proximal gradient algorithms, including the basic version and Nesterov's accelerated version.
 Then we give the theoretical analysis to the basic and Nesterov's accelerated versions.  The theoretical results show that our  inexact proximal gradient algorithms can have the same  convergence rates as the ones of exact proximal gradient algorithms  in the non-convex setting.
 Finally, we provide the   applications of our inexact proximal gradient algorithms on  three representative non-convex learning problems. The results confirm the  superiority of our inexact proximal gradient algorithms.

\onecolumn
\section{Appendix}
\subsection{Convergence Analysis} \label{convergence_analysis}
In this section, we  prove the convergence  rates of  our IPG and AIPG for the non-convex optimization. Specifically,  we first prove that IPG and AIPG converge to a critical point in the convex and non-convex optimization (Theorem \ref{theorem3})  if $ \sum_{k=1}^m  \varepsilon_k < \infty $. Next, we then prove that IPG has the convergence rate $O(\frac{1}{T})$ for the non-convex optimization (Theorem \ref{theorem2}) when the errors decrease at an appropriate rate.   Then, we  prove that the convergence rates for  AIPG in the non-convex optimization (Theorem \ref{theorem4}). In addition, because our AIPG is different to the one in \citep{schmidt2011convergence}, we also prove
the convergence rate $O(\frac{1}{T^2})$ for the convex optimization (Theorem \ref{theorem5}) when the errors decrease at an appropriate rate.

We first  prove that IPG and AIPG converge to a critical point for the convex or non-convex optimization (Theorem \ref{theorem3})  if $\{ \varepsilon_k\}$ is a decreasing sequence and $ \sum_{k=1}^m  \varepsilon_k < \infty $.
\begin{theorem} \label{theorem3}
 If $\{ \varepsilon_k\}$ is a decreasing sequence and  $ \sum_{k=1}^m  \varepsilon_k < \infty $, we have $\textbf{0} \in \lim_{k\rightarrow \infty} \nabla g(x_k) + \partial_{\varepsilon_k} h({x}_k)$ for IPG and AIPG in the convex and non-convex optimizations.
\end{theorem}
\begin{proof} We  prove that $\textbf{0} \in \lim_{k\rightarrow \infty} \nabla g(x_k) + \partial_{\varepsilon_k} h({x}_k)$ for AIPG in the convex and non-convex optimizations if $ \sum_{k=1}^m  \varepsilon_k < \infty $. The proof for IPG can be provided similarly.

 According to  line 5 in Algorithm 2 and (7), we have that
\begin{eqnarray} \label{thm1_1.9}
\left \langle \nabla g(x_{k}), v_{k+1} - x_{k}   \right \rangle + \frac{1}{2 \gamma} \left \| v_{k+1} - x_{k} \right \|^2 + h(v_{k+1}) \leq h(x_{k}) + \varepsilon_k
 \end{eqnarray}
Thus, we have that
\begin{eqnarray} \label{thm1_1.2}
&& f(v_{k+1}) = g(v_{k+1}) + h(v_{k+1})
\\ & \leq & \nonumber g(x_{k}) + \left \langle \nabla g(x_{k}) , v_{k+1} - x_{k}  \right \rangle + \frac{L}{2} \left \| v_{k+1} - x_{k} \right \|^2
\\ &  & \nonumber  + h(x_{k}) - \left \langle \nabla g(x_{k}), v_{k+1} - x_{k}   \right \rangle - \frac{1}{2 \gamma} \left \| v_{k+1} - x_{k} \right \|^2 + \varepsilon_k
\\ & = & \nonumber f(x_{k}) - \left ( \frac{1}{2 \gamma} - \frac{L}{2}  \right ) \left \| v_{k+1} - x_{k} \right \|^2 + \varepsilon_k
\end{eqnarray}
If $f(z_{k+1}) \leq f(v_{k+1})$, we have $x_{k+1} = z_{k+1}$ and $f(x_{k+1}) = f(z_{k+1}) \leq f(v_{k+1}) $. If $f(z_{k+1}) > f(v_{k+1})$, we have $x_{k+1} = v_{k+1}$ and $f(x_{k+1}) = f(v_{k+1}) $. Thus, we have that
\begin{eqnarray} \label{thm3_1.2}
 f(x_{k+1})
 \leq  f(x_{k}) -   \left ( \frac{1}{2 \gamma} - \frac{L}{2}  \right ) \left \| v_{k+1} - x_{k} \right \|^2 + \varepsilon_k
\end{eqnarray}
By summing
the the inequality (\ref{thm3_1.2}) over $k = 1,\cdots,m$,  we obtain
\begin{eqnarray} \label{thm3_1.3}
 f(x_{m+1}) \leq f(x_{0}) -  \left ( \frac{1}{2 \gamma} - \frac{L}{2}  \right ) \sum_{k=1}^{m} \left \| v_{k+1} -x_{k} \right \|^2  + \sum_{k=1}^{m}  \varepsilon_k
\end{eqnarray}
Same with the analysis for (\ref{thm1_1.3}) in Theorem \ref{equ_theorem1}, we have that
\begin{eqnarray} \label{thm3_1.4}
\sum_{k=1}^m \left \| v_{k+1} -x_{k} \right \|^2
\leq   \frac{1}{\frac{1}{2\gamma} - \frac{L}{2}}\left (  f(x_{0}) - f(x^*)  \right ) + \frac{1}{\frac{1}{2\gamma} - \frac{L}{2}} \sum_{k=1}^{m}  \varepsilon_k
 \end{eqnarray}
 We assume that $\sum_{k=1}^{m}  \varepsilon_k < \infty$. Thus, $\sum_{k=1}^m \left \| v_{k+1} -x_{k} \right \|^2 \leq  \frac{1}{\frac{1}{2\gamma} - \frac{L}{2}}\left (  f(x_{0}) - f(x^*)  \right ) + \frac{1}{\frac{1}{2\gamma} - \frac{L}{2}} \sum_{k=1}^{m}  \varepsilon_k < \infty$. So we have that
 \begin{eqnarray} \label{thm3_1.5}
\lim_{k\rightarrow \infty }\left \| v_{k+1} -x_{k} \right \|^2 = 0
 \end{eqnarray}
In addition, we have $\lim_{k\rightarrow \infty } \varepsilon_k =0$. Because $v_{k+1}$ is a $\varepsilon_k$-optimal solution to the proximal problem, according to Lemma 2 in \cite{schmidt2011convergence}, there exists $f_k$ such that $ \| f_k \| \leq \sqrt{2 \gamma \varepsilon_k } $ and
\begin{eqnarray}
\label{thm3_1.6}
0 &\in& \frac{1}{\gamma } \left (x_{k} -v_{k+1}  - \gamma  \nabla g(x_{k}) - f_k \right ) -  \partial_{\varepsilon_k} h(v_{k+1})
\\ & = & \nonumber \frac{1}{\gamma } \left (x_{k} -v_{k+1} - f_k \right )   -   \nabla g(x_{k}) +  \nabla g(v_{k+1}) - \nabla g(v_{k+1}) -  \partial_{\varepsilon_k} h(v_{k+1})
\end{eqnarray}
Thus, we have
\begin{eqnarray}
\label{thm3_1.7}
 \frac{1}{\gamma }\left (x_{k} -v_{k+1} - f_k \right )   -   \nabla g(x_{k}) +  \nabla g(v_{k+1})  \in  \nabla g(v_{k+1}) + \partial_{\varepsilon_k} h(v_{k+1})
\end{eqnarray}
In addition, we have
\begin{eqnarray}
\label{thm3_1.8}
 \left \|  \frac{1}{\gamma }\left (x_{k} -v_{k+1} - f_k \right )   -   \nabla g(x_{k}) +  \nabla g(v_{k+1})  \right \| \leq  \left (\frac{1}{\gamma } + L \right ) \left \| x_{k} -v_{k+1} \right \|  + \sqrt{ \frac{2 \varepsilon_k}{\gamma} }
\end{eqnarray}
Thus,  we have
\begin{eqnarray}
\label{thm3_1.8.1}
&& \lim_{k \rightarrow \infty}  \left \|  \frac{1}{\gamma }\left (x_{k} -v_{k+1} - f_k \right )   -   \nabla g(x_{k}) +  \nabla g(v_{k+1})  \right \|
\\ & \leq & \nonumber \lim_{k \rightarrow \infty} \left ( \left (\frac{1}{\gamma } + L \right ) \left \| x_{k} -v_{k+1} \right \|  + \sqrt{ \frac{2 \varepsilon_k}{\gamma} } \right ) = 0
\end{eqnarray}
Based on (\ref{thm3_1.7}) and (\ref{thm3_1.8.1}), we have that
\begin{eqnarray}
\label{thm3_1.9}
\textbf{0} \in \lim_{k\rightarrow \infty} \nabla g(v_{k}) + \partial_{\varepsilon_k} h(v_{k})
\end{eqnarray}
Because $\lim_{k\rightarrow \infty }\left \| v_{k+1} -x_{k} \right \|^2 = 0$ as proved in (\ref{thm3_1.5}), we have that
\begin{eqnarray}
\label{thm3_1.10}
\textbf{0} \in \lim_{k\rightarrow \infty} \nabla g(x_{k}) + \partial_{\varepsilon_k} h(x_{k})
\end{eqnarray}
  This completes the proof.
\end{proof}

\subsubsection{IPG for nonconvex optimization}
Based on (\ref{thm3_1.7}), we  similarly have
\begin{eqnarray}
\label{thm2_0.01}
\frac{1}{\gamma }\left (x_{k-1} -x_{k} - f_k \right )   -   \nabla g(x_{k-1}) +  \nabla g(x_{k})  \in  \nabla g(x_{k}) + \partial_{\varepsilon_k} h(x_{k})
\end{eqnarray}
for IPG. Further, similar with (\ref{thm3_1.8}), we have
\begin{eqnarray}
\label{thm2_0.02}
 \left \|  \frac{1}{\gamma }\left (x_{k-1} -x_{k} - f_k \right )   -   \nabla g(x_{k-1}) +  \nabla g(x_{k}) \right \| \leq  \left (\frac{1}{\gamma } + L \right ) \left \| x_{k} -x_{k-1} \right \|  + \sqrt{ \frac{2 \varepsilon_k}{\gamma} }
\end{eqnarray}
Thus, we have that
\begin{eqnarray}
\label{thm2_0.03}
\frac{1}{m}\sum_{k=1}^m  \min_{d_k \in \partial_{\varepsilon_k} h(x_{k})}  \left \| \nabla g(x_{k}) + d_k\right \|^2 \leq \frac{1}{m}\sum_{k=1}^m  \left (  \left (\frac{1}{\gamma } + L \right ) \left \| x_{k} -x_{k-1} \right \|  + \sqrt{ \frac{2 \varepsilon_k}{\gamma} } \right )
\end{eqnarray}
Based on (\ref{thm2_0.03}), we use $ \frac{1}{m}\sum_{k=1}^m \left \| x_k - x_{k-1} \right \|^2$ to analyze the convergence rate in the non-convex setting.
\begin{theorem} \label{theorem2}
For $g(x)$ is nonconvex, and $h(x)$ is convex or nonconvex, we have the  following results for  IPG:
\begin{enumerate}
\item If  $h(x)$ is convex,  we have that
\begin{eqnarray}\label{equ_theorem1}
 \frac{1}{m}\sum_{k=1}^m \left \| x_k - x_{k-1} \right \|^2
\leq \frac{1}{m}\left ( 2A_m +  \sqrt{ \frac{1}{\frac{1}{\gamma} - \frac{L}{2}} \left (  f(x_{0}) - f(x^*)  \right )} + \sqrt{B_m}  \right )^2
\end{eqnarray}
where $A_m = { \frac{1}{2}\sum_{k=1}^m {\frac{1}{\frac{1}{\gamma} - \frac{L}{2}} \sqrt{\frac{2  \varepsilon_k}{\gamma} }} }$ and $B_m= {\frac{1}{\frac{1}{\gamma} - \frac{L}{2}} \sum_{k=1}^m  \varepsilon_k }$.
\item If  $h(x)$ is non-convex,  we have that
\begin{eqnarray}\label{equ_theorem2}
 \frac{1}{m}\sum_{k=1}^m \left \| x_k - x_{k-1} \right \|^2 \leq {\frac{1}{ m \left (\frac{1}{2\gamma} - \frac{L}{2} \right )}  } \left ( f(x_{0}) -  f(x^*)  +  \sum_{k=1}^m  \varepsilon_k \right )
\end{eqnarray}
\end{enumerate}
\end{theorem}
\begin{proof} We first give the proof for the case that $h(x)$ is convex. Since $x_k \in \textrm{Prox}^{\varepsilon_k}_{\gamma g} \left (x_{t-1} - \gamma  \nabla g(x_{k-1})  \right )$, according to Lemma 2 in \citep{schmidt2011convergence}, there exists $f_k$ such that $ \| f_k \| \leq \sqrt{ 2 \gamma \varepsilon_k } $ and
\begin{eqnarray}
\label{thm1_1.02}
\frac{1}{\gamma } \left (x_{k-1} -x_k  - \gamma  \nabla g(x_{k-1}) - f_k \right ) \in \partial_{\varepsilon_k} h(x_k)
\end{eqnarray}
We have that
\begin{eqnarray} \label{thm1_1.2}
&& f(x_k) = g(x_k) + h(x_k)
\\ & \leq & \nonumber g(x_{k-1}) + \left \langle \nabla g(x_{k-1}) , x_k - x_{k-1}  \right \rangle + \frac{L}{2} \left \| x_k - x_{k-1} \right \|^2
\\ &  & \nonumber  + h(x_{k-1}) - \left \langle \nabla g(x_{k-1}) + \frac{1}{\gamma }(x_k - x_{k-1} + f_i) , x_k - x_{k-1}  \right \rangle + \varepsilon_k
\\ & = & \nonumber f(x_{k-1}) - \frac{1}{\gamma }\left \| x_k - x_{k-1} \right \|^2 + \frac{L}{2} \left \| x_k - x_{k-1} \right \|^2 - \left \langle  \frac{1}{\gamma} f_i , x_k - x_{k-1}  \right \rangle + \varepsilon_k
\\ & \leq & \nonumber f(x_{k-1}) -  \left ( \frac{1}{\gamma} - \frac{L}{2} \right ) \left \| x_k - x_{k-1} \right \|^2 + \sqrt{\frac{2  \varepsilon_k}{\gamma} } \left \| x_k - x_{k-1}  \right \| + \varepsilon_k
\end{eqnarray}
where the first inequality uses (4),  the convexity of $h$ and (\ref{thm1_1.02}).
By summing
the the inequality (\ref{thm1_1.2}) over $k = 1,\cdots,m$,  we obtain
\begin{eqnarray} \label{thm1_1.3}
 f(x_m) \leq f(x_{0}) -  \left ( \frac{1}{\gamma} - \frac{L}{2} \right ) \sum_{k=1}^m \left \| x_k - x_{k-1} \right \|^2 + \sum_{k=1}^m  \sqrt{\frac{2  \varepsilon_k}{\gamma} } \left \| x_k - x_{k-1}  \right \| + \sum_{k=1}^m  \varepsilon_k
\end{eqnarray}
According to (\ref{thm1_1.3}), we have that
\begin{eqnarray} \label{thm1_1.4}
 \left \| x_m - x_{m-1} \right \|^2  \leq  \underbrace{\frac{1}{\frac{1}{\gamma} - \frac{L}{2}} \left (  f(x_{0}) - f(x_m) + \sum_{k=1}^m  \varepsilon_k \right )}_{A} + \sum_{k=1}^m   \underbrace{\frac{1}{\frac{1}{\gamma} - \frac{L}{2}} \sqrt{\frac{2  \varepsilon_k}{\gamma} }}_{\lambda_k}  \left \| x_k - x_{k-1}  \right \|
\end{eqnarray}
According to Lemma 1 in \citep{schmidt2011convergence}, we have that
\begin{eqnarray} \label{thm1_1.5}
 && \left \| x_m - x_{m-1} \right \|
 \\ & \leq & \nonumber  {\frac{1}{2}\sum_{k=1}^m {\lambda_k} } + \left ( A+ \left ( \frac{1}{2}\sum_{k=1}^m {\lambda_k} \right )^2 \right )^{\frac{1}{2}}
 \\ & = & \nonumber \underbrace{ \frac{1}{2}\sum_{k=1}^m {\frac{1}{\frac{1}{\gamma} - \frac{L}{2}} \sqrt{\frac{2  \varepsilon_k}{\gamma} }} }_{A_m} + \left ( \frac{1}{\frac{1}{\gamma} - \frac{L}{2}} \left (  f(x_{0}) - f(x_m)  \right )+ \underbrace{\frac{1}{\frac{1}{\gamma} - \frac{L}{2}} \sum_{k=1}^m  \varepsilon_k }_{B_m} + \left ( \underbrace{ \frac{1}{2}\sum_{k=1}^m {\frac{1}{\frac{1}{\gamma} - \frac{L}{2}} \sqrt{\frac{2  \varepsilon_k}{\gamma} }} }_{A_m} \right )^2 \right )^{\frac{1}{2}}
 \\ & \leq & \nonumber A_m + \left ( \frac{1}{\frac{1}{\gamma} - \frac{L}{2}} \left (  f(x_{0}) - f(x^*)\right )  + B_m + A_m^2  \right )^{\frac{1}{2}}
\end{eqnarray}
Because $A_k$ and $B_k$ are increasing sequences,  $\forall k \leq m$, we have that
\begin{eqnarray} \label{thm1_1.6}
 && \left \| x_k - x_{k-1} \right \|
 \\ & \leq & \nonumber A_m + \left ( \frac{1}{\frac{1}{\gamma} - \frac{L}{2}} \left (  f(x_{0}) - f(x^*) \right ) + B_m + A_m^2  \right )^{\frac{1}{2}}
  \\ & \leq & \nonumber A_m +  \sqrt{\frac{1}{\frac{1}{\gamma} - \frac{L}{2}} \left (  f(x_{0}) - f(x^*)  \right )} + \sqrt{B_m} + A_m
 \\ & \leq & \nonumber 2A_m +  \sqrt{\frac{1}{\frac{1}{\gamma} - \frac{L}{2}} \left (  f(x_{0}) - f(x^*)  \right )} + \sqrt{B_m}
 \end{eqnarray}
 According to (\ref{thm1_1.3}) and (\ref{thm1_1.6}),  we have that
 \begin{eqnarray} \label{thm1_1.7}
&&  \sum_{k=1}^m \left  \| x_k - x_{k-1} \right \|^2
 \\ & \leq & \nonumber  \frac{1}{\frac{1}{\gamma} - \frac{L}{2}}  \left ( f(x_{0}) - f(x_m) \right ) + \frac{1}{\frac{1}{\gamma} - \frac{L}{2}}  \sum_{k=1}^m  \varepsilon_k + \frac{1}{\frac{1}{\gamma} - \frac{L}{2}} \sum_{k=1}^m  \sqrt{{2 L \varepsilon_k }}  \left \| x_k - x_{k-1}  \right \|
\\ & \leq & \nonumber     \frac{1}{\frac{1}{\gamma} - \frac{L}{2}}  \left (  f(x_{0}) - f(x^*) \right ) + B_m +  2A_m \left ( 2A_m +  \sqrt{\frac{1}{\frac{1}{\gamma} - \frac{L}{2}} \left (  f(x_{0}) - f(x^*)  \right )} + \sqrt{B_m}  \right )
\\ & \leq & \nonumber     \frac{1}{\frac{1}{\gamma} - \frac{L}{2}}  \left (  f(x_{0}) - f(x^*) \right ) +2 \sqrt{B_m}  \sqrt{\frac{1}{\frac{1}{\gamma} - \frac{L}{2}} \left (  f(x_{0}) - f(x^*)  \right )}  + B_m
\\ &  & \nonumber  +  2A_m \left ( 2A_m +  \sqrt{ \frac{1}{\frac{1}{\gamma} - \frac{L}{2}} \left (  f(x_{0}) - f(x^*)  \right )} + \sqrt{B_m}  \right )
\\ & =  & \nonumber \left ( 2A_m +  \sqrt{ \frac{1}{\frac{1}{\gamma} - \frac{L}{2}} \left (  f(x_{0}) - f(x^*)  \right )} + \sqrt{B_m}  \right )^2
\end{eqnarray}
Based on (\ref{thm1_1.7}), we have that
\begin{eqnarray} \label{thm1_1.8}
 \frac{1}{m}\sum_{k=1}^m \left \| x_k - x_{k-1} \right \|^2
\leq \frac{1}{m}\left ( 2A_m +  \sqrt{\frac{1}{\frac{1}{\gamma} - \frac{L}{2}} \left (  f(x_{0}) - f(x^*)  \right )} + \sqrt{B_m}  \right )^2
 \end{eqnarray}
 This completes the conclusion for the case that $h(x)$ is convex.

 Next, we give the the proof for the case that $h(x)$ is non-convex. According to  line 3 in Algorithm 1 and (7), we have that
\begin{eqnarray} \label{thm1_1.9}
\left \langle \nabla g(x_{k-1}), x_{k} - x_{k-1}   \right \rangle + \frac{1}{2 \gamma} \left \| x_{k} - x_{k-1} \right \|^2 + h(x_{k}) \leq h(x_{k-1}) + \varepsilon_k
 \end{eqnarray}
Thus, we have that
\begin{eqnarray} \label{thm1_1.10}
&& f(x_k) = g(x_k) + h(x_k)
\\ & \leq & \nonumber g(x_{k-1}) + \left \langle \nabla g(x_{k-1}) , x_k - x_{k-1}  \right \rangle + \frac{L}{2} \left \| x_k - x_{k-1} \right \|^2
\\ &  & \nonumber  + h(x_{k-1}) - \left \langle \nabla g(x_{k-1}), x_{k} - x_{k-1}   \right \rangle - \frac{1}{2 \gamma} \left \| x_{k} - x_{k-1} \right \|^2 + \varepsilon_k
\\ & = & \nonumber f(x_{k-1}) - \left ( \frac{1}{2 \gamma} - \frac{L}{2}  \right ) \left \| x_k - x_{k-1} \right \|^2 + \varepsilon_k
\end{eqnarray}
By summing
the the inequality (\ref{thm1_1.10}) over $k = 1,\cdots,m$,  we obtain
\begin{eqnarray} \label{thm1_1.11}
 f(x_m) \leq f(x_{0}) -  \left ( \frac{1}{2 \gamma} - \frac{L}{2}  \right ) \sum_{k=1}^m \left \| x_k - x_{k-1} \right \|^2 +  \sum_{k=1}^m  \varepsilon_k
\end{eqnarray}
Based on (\ref{thm1_1.11}), we have that
\begin{eqnarray} \label{thm1_1.12}
 \frac{1}{m}\sum_{k=1}^m \left \| x_k - x_{k-1} \right \|^2 \leq {\frac{1}{ m \left (\frac{1}{2\gamma} - \frac{L}{2} \right )}  } \left ( f(x_{0}) -  f(x^*)  +  \sum_{k=1}^m  \varepsilon_k \right )
\end{eqnarray}
 This completes the proof.
\end{proof}
\begin{remark}
Theorem \ref{theorem2} implies that IPG has the convergence rate $O(\frac{1}{T})$ for the non-convex optimization without errors. If  $\{ \sqrt{\varepsilon_k} \}$ is summable and $h(x)$ is  convex, we can also have that IPG has the convergence rate $O(\frac{1}{T})$ for the non-convex optimization. If  $\{ \varepsilon_k \}$ is summable and $h(x)$ is  non-convex, we can also have that IPG has the convergence rate $O(\frac{1}{T})$ for the non-convex optimization.
\end{remark}
\subsubsection{AIPG}
In this section,   we  prove that the convergence rates for  AIPG in the non-convex optimization (Theorem \ref{theorem4}). In addition,  we  prove
the convergence rate $O(\frac{1}{T^2})$ for the convex optimization (Theorem \ref{theorem5}) when the errors decrease at an appropriate rate.

\noindent \textbf{Nonconvex optimization}
To prove the convergence rate of AIPG for nonconvex optimization, we first give   Lemma \ref{lemma1}. Lemma \ref{lemma1} is an  $\varepsilon$ approximate version of uniformized KL property which  can be proved similarly with the analysis of Lemma 6 in \citep{bolte2014proximal}. Based on Lemma \ref{lemma1}, we prove the convergence rate of AIPG for non-convex optimization (Theorem \ref{theorem4}).
\begin{lemma} \label{lemma1}
Let $\Omega$ be a compact set and let $f(x) : \mathbb{R}^N \rightarrow (-\infty, +\infty ]$  be a proper and lower semicontinuous function. Assume that $f(x)$ is constant on $\Omega$ and satisfies the $\varepsilon$-KL property at each point of $\Omega$. Then there exists $\epsilon>0$, $\eta>0$ and $\varphi \in \Phi_{\eta}$, such that for all $\overline{u} \in \Omega$ and all $u$ in the following intersection
\begin{eqnarray}\label{lemma_equ1}
\{u \in \mathbb{R}^N: dist(u,\Omega) <\epsilon \} \cap \{u \in \mathbb{R}^N:  f(\overline{u}) < f({u}) <  f(\overline{u})  + \eta \}
\end{eqnarray}
the following inequality holds
\begin{eqnarray} \label{lemma_equ2}
\varphi'(f(u) - f(\overline{u})) dist (\textbf{0}, \nabla g(u) + \partial_{\varepsilon} h(u))) \geq 1
\end{eqnarray}
\end{lemma}

\begin{theorem} \label{theorem4}
Assume that $g$ is a nonconvex function with Lipschitz continuous gradients, $h$ is a proper and lower semicontinuous function. If  the function $f$ satisfies the $\varepsilon$-KL property, $\varepsilon_k = \alpha \left \| v_{k+1} - x_{k} \right \|^2$, $\alpha \geq 0$, $ \frac{1}{2 \gamma} - \frac{L}{2} -  \alpha \geq 0$ and the desingularising function has the form $\varphi(t)=\frac{C}{\theta} t^\theta$ for some $C>0$, $\theta \in (0,1]$, then
\begin{enumerate}
\item If $\theta=1$, there exists $k_1$ such that $f(x_k)= f^* $ for all $k>k_1$ and AIPG terminates in a finite number of steps, where $\lim_{k \rightarrow \infty } f(x_k) = f^*$.
\item If $\theta \in [\frac{1}{2},1)$, there exists $k_2$ such that  for all $k>k_2$
\begin{eqnarray} \label{thm4_1.1}
f(x_{k}) - \lim_{k\rightarrow \infty} f(x_k)\leq  \left ( \frac{d_1C^2}{1+d_1 C^2} \right )^{k-k_2} \left (f(v_{k}) - f^* \right )
\end{eqnarray}
where $d_1 =  \frac{\left  ( \frac{1}{\gamma } + L  + \sqrt{ \frac{2 \alpha}{\gamma} } \right )^2}{  \frac{1}{2 \gamma} - \frac{L}{2} -  \alpha  }$.
\item If $\theta \in (0,\frac{1}{2})$, there exists $k_3$ such that  for all $k>k_3$
\begin{eqnarray} \label{thm4_1.1}
f(x_{k}) - \lim_{k\rightarrow \infty} f(x_k)\leq  \left ( \frac{C}{(k-k_3)d_2 (1-2 \theta)} \right )^{\frac{1}{1-2 \theta}}
\end{eqnarray}
where $d_2 =  \min \left  \{ \frac{1}{2d_1 C}, \frac{C}{1-2\theta} \left ( 2^{\frac{2 \theta -1}{2 \theta -2}} -1 \right ) \left (f(v_{0}) - f^* \right )^{2 \theta -1}\right  \}$.
\end{enumerate}
\end{theorem}
\begin{proof}
We first give the upper bounds for $\left \| v_{k+1} -x_{k} \right \|^2$ and $\textrm{dist}(\textbf{0},  \nabla g(v_{k+1}) + \partial_{\varepsilon_k} h(v_{k+1}))$ respectively.
From (\ref{thm1_1.2}) and $f(x_k) \leq f(v_k)$, we have that
\begin{eqnarray} \label{thm4_1.1}
 f(v_{k+1})
 &\leq&  f(v_{k}) -  \left ( \frac{1}{2 \gamma} - \frac{L}{2}  \right ) \left \| v_{k+1} - x_{k} \right \|^2 + \varepsilon_k
 \\ & = & \nonumber f(v_{k}) - \left  (   \frac{1}{2 \gamma} - \frac{L}{2} -  \alpha \right )\left \| v_{k+1} -x_{k} \right \|^2
\end{eqnarray}
Thus, we have that $ f(v_{k+1}) \leq f(v_{k})$ and
\begin{eqnarray} \label{thm4_1.2}
\left \| v_{k+1} -x_{k} \right \|^2
\leq  \frac{f(v_{k}) -   f(v_{k+1})}{  \frac{1}{2 \gamma} - \frac{L}{2} -  \alpha  }
\end{eqnarray}
From (\ref{thm3_1.7}), we have that
\begin{eqnarray} \label{thm4_1.3}
 \textrm{dist}(\textbf{0},  \nabla g(v_{k+1}) + \partial_{\varepsilon_k} h(v_{k+1}))
 & \leq &  \left (\frac{1}{\gamma } + L \right ) \left \| x_{k} -v_{k+1} \right \|  + \sqrt{ \frac{2 \varepsilon_k}{\gamma} }
 \\ & = & \nonumber \left  ( \frac{1}{\gamma } + L  + \sqrt{ \frac{2 \alpha}{\gamma} } \right ) \left \| x_{k} -v_{k+1} \right \|
\end{eqnarray}
According to (\ref{thm3_1.5}), we known $\{x_k \}$ and $\{v_k \}$ convergence to the same points. Let $\Omega$ be the set that contains all the convergence points of $\{x_k \}$ (also $\{v_k \}$). Because $ f(v_{k+1}) \leq f(v_{k})$, $\{f(v_{k})\}$ is a monotonically decreasing sequence. Thus, $f(v_{k})$ has the same value at all the convergence points in $\Omega$, which is denoted as $f^*$.

%If there exists $\overline{k}$ such that $f(v^{\overline{k}}) = f^*$, we have that $f(v^{\overline{k}}) = f(v^{\overline{k}}+1) = \cdots = f^*$. Similarly, we have $\left \| v_{\overline{k}+1} -v_{\overline{k}} \right \| = \left \| v_{\overline{k}+2} -v_{\overline{k}+1}  \right \| = \cdots = 0$.
 Because  $\{f(v_{k})\}$ is a monotonically decreasing sequence, there exists $\widehat{k}_1$ such that $f(v_k) < f^* + \eta$, $\forall k > \widehat{k}_1$. On the other hand, because $\lim_{k\rightarrow \infty}dis(v_k, \Omega) = 0$, there exists $\widehat{k}_2$ such that $ dis(v_k, \Omega) < \epsilon$, $\forall k > \widehat{k}_2$. Let $k>k_0 = \max \{\widehat{k}_1, \widehat{k}_2 \}$, we have
 \begin{eqnarray}\label{thm4_1.3.1}
v_k \in \{ v: dist(v,\Omega) < \epsilon \} \cap \{v:  f^* < f(v) <  f^*  + \eta \}
\end{eqnarray}
From Lemma \ref{lemma1}, there exists a concave function $\varphi$ such that
\begin{eqnarray} \label{thm4_1.3.2}
\varphi'(f(v_k) - f^* ) dist (\textbf{0}, \nabla g(v_k) + \partial_{\varepsilon} h(v_k))) \geq 1
\end{eqnarray}
Define $r_k= f(v_k) - f^*$.  According to (\ref{thm4_1.2}), (\ref{thm4_1.3}) and (\ref{thm4_1.3.2}), we have that
\begin{eqnarray} \label{thm4_1.4}
 1 & \leq &  \left ( \varphi'(f(v_k) - f^*) dist (\textbf{0}, \nabla g(u) + \partial_{\varepsilon} h(v_k))  \right )^2
\\  & \leq & \nonumber \left (  \varphi'( r_k ) \right )^2  \left  ( \frac{1}{\gamma } + L  + \sqrt{ \frac{2 \alpha}{\gamma} } \right )^2 \left \| x_{k-1} -v_{k} \right \|^2
\\  & \leq & \nonumber \left (  \varphi'( r_k ) \right )^2  \left  ( \frac{1}{\gamma } + L  + \sqrt{ \frac{2 \alpha}{\gamma} } \right )^2  \frac{f(v_{k-1}) -   f(v_{k})}{  \frac{1}{2 \gamma} - \frac{L}{2} -  \alpha  }
\\  & = & \nonumber d_1 \left (  \varphi'( r_k ) \right )^2 \left ( r_{k-1} - r_k \right )
\end{eqnarray}
for all $k>k_0$, where $d_1 =  \frac{\left  ( \frac{1}{\gamma } + L  + \sqrt{ \frac{2 \alpha}{\gamma} } \right )^2}{  \frac{1}{2 \gamma} - \frac{L}{2} -  \alpha  }$. Because $\varphi$ has the form of $\varphi(t)=\frac{C}{\theta} t^\theta$, we have $\varphi'(t)= C t^{\theta-1}$. Thus, according to (\ref{thm4_1.4}), we have that
\begin{eqnarray} \label{thm4_1.5}
 1 \leq d_1C^2 t^{2\theta-2}\left ( r_{k-1} - r_k \right )
\end{eqnarray}
Next, we consider the three cases, i.e., $\theta=1$, $\theta \in [\frac{1}{2},1)$ and $\theta \in (0,\frac{1}{2})$, which are also considered in \citep{li2015accelerated}.
Same with the analysis  of Theorem 3 in \citep{li2015accelerated}, we can have the conclusions in Theorem \ref{theorem4}.
 This completes the proof.
\end{proof}

\noindent \textbf{Convex optimization}
In this section, we  prove
the convergence rate $O(\frac{1}{T^2})$ for the convex optimization (Theorem \ref{theorem5}) when the errors decrease at an appropriate rate.
\begin{theorem} \label{theorem5}
Assume that $f$ is convex. For AIPG, we have that
\begin{eqnarray}\label{equ_theorem2_o.1}
f(x_{k+1}) - f(x^*)\leq \frac{2L}{(m+1)^2 } \left ( \left \| x_0 - x^* \right \|  + 2A_m + \sqrt{B_m} \right )^2
\end{eqnarray}
where $A_m = { \frac{1}{2}\sum_{k=1}^m {2 \gamma t_k\sqrt{{2 \frac{1}{\gamma} \varepsilon_k }}} }$, $B_m= { 2 \gamma \sum_{k=1}^m  t_k^2\varepsilon_k}$.
\end{theorem}
\begin{proof}
We have that
\begin{eqnarray} \label{thm2_1.2}
&& f(z_{k+1}) = g(z_{k+1}) + h(z_{k+1})
\\ & \leq & \nonumber g(y_k) + \left \langle \nabla g(y_{k}) , z_{k+1} - y_k  \right \rangle + \frac{L}{2} \left \| z_{k+1} - y_k \right \|^2  + h(z_{k+1})
\\ & =  & \nonumber  g(y_k) + \left \langle \nabla g(y_{k}) ,x - y_k  \right \rangle +  \left \langle \nabla g(y_{k}) , z_{k+1} - x  \right \rangle + \frac{L}{2} \left \| z_{k+1} - y_k \right \|^2  + h(z_{k+1})
\\ & \leq & \nonumber  g(x)  +  \left \langle \nabla g(y_{k}) , z_{k+1} - x  \right \rangle + \frac{L}{2} \left \| z_{k+1} - y_k \right \|^2  + h(z_{k+1})
\\ & \leq & \nonumber  g(x)  +  \left \langle \nabla g(y_{k}) , z_{k+1} - x  \right \rangle + \frac{L}{2} \left \| z_{k+1} - y_k \right \|^2
\\ &  & \nonumber  + h(x)  - \left \langle \nabla g(y_k) + \frac{1}{\gamma}(z_{k+1} - y_{k} + f_k) , z_{k+1} - x \right \rangle + \varepsilon_k
\\ & = & \nonumber  f(x)   + \frac{L}{2} \left \| z_{k+1} - y_k \right \|^2
  - \left \langle  \frac{1}{\gamma}(z_{k+1} - y_{k}) , z_{k+1}- y_{k} +  y_{k} - x \right \rangle + \left \langle  \frac{1}{\gamma} f_k , x - z_{k+1}   \right \rangle  + \varepsilon_k
\\ & = & \nonumber  f(x)  - \left ( \frac{1}{\gamma} - \frac{L}{2} \right ) \left \| z_{k+1} - y_k \right \|^2
  - \left \langle  \frac{1}{\gamma} (z_{k+1} - y_{k} ) ,   y_{k} - x \right \rangle + \left \langle  \frac{1}{\gamma} f_k , x - z_{k+1}   \right \rangle  + \varepsilon_k
\\ & \leq & \nonumber  f(x)  - \frac{1}{2\gamma} \left \| z_{k+1} - y_k \right \|^2
  - \left \langle  \frac{1}{\gamma} (z_{k+1} - y_{k} ) ,   y_{k} - x \right \rangle + \left \langle  \frac{1}{\gamma} f_k , x - z_{k+1}   \right \rangle  + \varepsilon_k
\end{eqnarray}
where the first inequality uses (4), the second inequality uses the  convexity of $g(x)$, the third inequality uses  the convexity of $h(x)$ and (\ref{thm1_1.02}), the final inequality uses $\gamma < \frac{1}{L}$.
Let $x=x_k$ and $x^*$, we have
\begin{eqnarray} \label{thm2_1.3}
f(z_{k+1}) - f(x_k) \leq - \frac{1}{2\gamma} \left \| z_{k+1} - y_k \right \|^2
  - \left \langle  \frac{1}{\gamma}(z_{k+1} - y_{k} ) ,   y_{k} - x_k \right \rangle + \left \langle  \frac{1}{\gamma} f_k , x_k- z_{k+1}   \right \rangle  + \varepsilon_k
\\   \label{thm2_1.4} f(z_{k+1}) - f(x^*) \leq - \frac{1}{2\gamma} \left \| z_{k+1} - y_k \right \|^2
  - \left \langle  \frac{1}{\gamma}(z_{k+1} - y_{k} ) ,   y_{k} - x^* \right \rangle + \left \langle  \frac{1}{\gamma} f_k ,x^*- z_{k+1}   \right \rangle  + \varepsilon_k
\end{eqnarray}
Multiplying (\ref{thm2_1.3}) by $t_k-1$ and adding (\ref{thm2_1.4}), we have
\begin{eqnarray} \label{thm2_1.5}
&& t_k f(z_{k+1}) - (t_k-1) f(x_k) - f(x^*)
\\ & \leq & \nonumber - \frac{ t_k}{2 \gamma} \left \| z_{k+1} - y_k \right \|^2  - \left \langle  \frac{1}{\gamma}(z_{k+1} - y_{k} ) ,  (t_k-1) \left ( y_{k} - x_k \right ) +  y_{k} - x^* \right \rangle
\\ &  & \nonumber + \left \langle  \frac{1}{\gamma} f_k , (t_k-1) \left ( x_k- z_{k+1} \right ) + x^*- z_{k+1}   \right \rangle  + t_k \varepsilon_k
\end{eqnarray}
Thus, we have
\begin{eqnarray} \label{thm2_1.6}
&& t_k \left  (f(z_{k+1}) - f(x^*) \right  )- (t_k-1) \left  ( f(x_k) - f(x^*) \right )
\\ & \leq & \nonumber - \frac{ t_k}{2 \gamma} \left \| z_{k+1} - y_k \right \|^2  - \left \langle  \frac{1}{\gamma}(z_{k+1} - y_{k} ) ,  (t_k-1) \left ( y_{k} - x_k \right ) +  y_{k} - x^* \right \rangle
\\ &  & \nonumber + \left \langle  \frac{1}{\gamma} f_k , (t_k-1) \left ( x_k- z_{k+1} \right ) + x^*- z_{k+1}   \right \rangle  + t_k \varepsilon_k
\end{eqnarray}
 Multiplying both sides of (\ref{thm2_1.6}) by $t_k$ and using $t_k^2  - t_k = (t_{k-1})^2$ in Algorithm \ref{algorithm2}, we have that
 \begin{eqnarray} \label{thm2_1.7}
&& t_k^2 \left  (f(z_{k+1}) - f(x^*) \right  )- t_{k-1}^2 \left  ( f(x_k) - f(x^*) \right )
\\ & \leq & \nonumber - \frac{ t_k^2}{2 \gamma} \left \| z_{k+1} - y_k \right \|^2  - \left \langle t_k \frac{1}{\gamma}(z_{k+1} - y_{k} ) ,  (t_k-1) \left ( y_{k} - x_k \right ) +  y_{k} - x^* \right \rangle
\\ &  & \nonumber + \left \langle t_k \frac{1}{\gamma} f_k , (t_k-1) \left ( x_k- z_{k+1} \right ) + x^*- z_{k+1}   \right \rangle  + t_k^2 \varepsilon_k
\\ & = & \nonumber - \frac{ t_k^2}{2 \gamma} \left \| z_{k+1} - y_k \right \|^2  - \left \langle t_k \frac{1}{\gamma}(z_{k+1} - y_{k} ) ,  t_k y_{k}   -  (t_k-1)  x_k  - x^* \right \rangle
\\ &  & \nonumber + \left \langle t_k \frac{1}{\gamma} f_k , (t_k-1)  x_k  - t_k z_{k+1} + x^*   \right \rangle  + t_k^2 \varepsilon_k
\\ & = & \nonumber \frac{1}{2\gamma } \left ( \left \| (t_k-1)x_k -t_k y_{k}  + x^*   \right \|^2 - \left \| (t_k-1)x_k -t_k z_{k+1}  + x^*  \right \|^2  \right )
\\ &  & \nonumber + \left \langle t_k \frac{1}{\gamma} f_k , (t_k-1)  x_k  - t_k z_{k+1} + x^*   \right \rangle  + t_k^2 \varepsilon_k
\end{eqnarray}
Define $U_{k+1} = t_k z_{k+1} - (t_k-1) x_k -  x^* $. Because $y_k = x_k + \frac{t_{k-1}}{t_k} (z_k - x_k ) + \frac{t_{k-1 } -1 }{ t_k} (x_k - x_{k-1})$, we have that \begin{eqnarray}\label{thm2_1.8}
z_{k} = \frac{t_k}{t_{k-1}}y_k +  \frac{1 - t_k}{t_{k-1}} x_k + \frac{t_{k-1} - 1}{t_{k-1}}x_{k-1} \end{eqnarray}
  Thus, we have that  $U_{k} = t_{k-1} z_{k} - (t_{k-1}-1) x_{k-1} -  x^* = t_ky_k - (t_{k-1} - 1) x_k - x^* $. Substitute  $U_{k+1}$ and $U_{k}$ into (\ref{thm2_1.7}), we have that
 \begin{eqnarray} \label{thm2_1.9}
&&  t_k^2 \left  (f(z_{k+1}) - f(x^*) \right  )- t_{k-1}^2 \left  ( f(x_k) - f(x^*) \right )
\\ & \leq & \nonumber \frac{1 }{2 \gamma}  \left ( \left \| U_{k} \right \|^2 -  \left \| U_{k+1} \right \|^2 \right ) - \left \langle t_k \frac{1}{\gamma} f_k ,  U_{k+1}  \right \rangle  + t_k^2 \varepsilon_k
\end{eqnarray}
If $f(z_{k+1}) \leq f(v_{k+1})$, we have $x_{k+1} = z_{k+1}$. Thus,
 \begin{eqnarray} \label{thm2_1.10}
&&  t_k^2 \left  (f(x_{k+1}) - f(x^*) \right  )- t_{k-1}^2 \left  ( f(x_k) - f(x^*) \right )
\\ & = & \nonumber t_k^2 \left  (f(z_{k+1}) - f(x^*) \right  )- t_{k-1}^2 \left  ( f(x_k) - f(x^*) \right )
\\ & \leq & \nonumber \frac{1 }{2 \gamma}  \left ( \left \| U_{k} \right \|^2 -  \left \| U_{k+1} \right \|^2 \right ) - \left \langle t_k \frac{1}{\gamma} f_k ,  U_{k+1}  \right \rangle  + t_k^2 \varepsilon_k
\end{eqnarray}
If $f(z_{k+1}) > f(v_{k+1})$, we have $x_{k+1} = v_{k+1}$. Thus,
 \begin{eqnarray} \label{thm2_1.11}
&&  t_k^2 \left  (f(x_{k+1}) - f(x^*) \right  )- t_{k-1}^2 \left  ( f(x_k) - f(x^*) \right )
\\ & = & \nonumber t_k^2 \left  (f(z_{k+1}) - f(x^*) \right  )- t_{k-1}^2 \left  ( f(x_k) - f(x^*) \right )
\\ & \leq & \nonumber  \frac{1 }{2 \gamma}  \left ( \left \| U_{k} \right \|^2 -  \left \| U_{k+1} \right \|^2 \right ) - \left \langle t_k \frac{1}{\gamma} f_k ,  U_{k+1}  \right \rangle  + t_k^2 \varepsilon_k
\end{eqnarray}
Combining (\ref{thm2_1.10}) and (\ref{thm2_1.11}), we have
 \begin{eqnarray} \label{thm2_1.12}
&&  t_k^2 \left  (f(x_{k+1}) - f(x^*) \right  )- t_{k-1}^2 \left  ( f(x_k) - f(x^*) \right )
\\ & \leq & \nonumber \frac{1 }{2 \gamma} \left ( \left \| U_{k} \right \|^2 -  \left \| U_{k+1} \right \|^2 \right ) - \left \langle t_k \frac{1}{\gamma} f_k ,  U_{k+1}  \right \rangle  + t_k^2 \varepsilon_k
\\ & \leq & \nonumber \frac{1 }{2 \gamma}  \left ( \left \| U_{k} \right \|^2 -  \left \| U_{k+1} \right \|^2 \right ) + t_k\sqrt{{2 \frac{1}{\gamma} \varepsilon_k }} \left \|U_{k+1}  \right \|  + t_k^2 \varepsilon_k
\end{eqnarray}
By summing
the the inequality (\ref{thm2_1.12}) over $k = 1,\cdots,m$,  we obtain
 \begin{eqnarray} \label{thm2_1.13}
 &&  t_m^2 \left  (f(x_{k+1}) - f(x^*) \right  )
\\ & = & \nonumber t_m^2 \left  (f(x_{k+1}) - f(x^*) \right  )- t_{0}^2 \left  ( f(x_1) - f(x^*) \right )
\\ & \leq & \nonumber \frac{1 }{2 \gamma}  \left ( \left \| U_{1} \right \|^2 -  \left \| U_{m+1} \right \|^2 \right ) + \sum_{k=1}^m t_k\sqrt{{2 \frac{1}{\gamma} \varepsilon_k }} \left \|U_{k+1}  \right \|  + \sum_{k=1}^m  t_k^2 \varepsilon_k
\end{eqnarray}
According to (\ref{thm2_1.13}), we have that
 \begin{eqnarray} \label{thm2_1.14}
  \left \| U_{m+1} \right \|^2
 \leq \underbrace{ \left \| U_{1} \right \|^2  + 2 \gamma \sum_{k=1}^m  t_k^2\varepsilon_k }_{A} +  \sum_{k=1}^m  \underbrace{2 \gamma t_k\sqrt{{2 \frac{1}{\gamma} \varepsilon_k }}}_{\lambda_k} \left \|U_{k+1}  \right \|
\end{eqnarray}
According to Lemma 1 in \citep{schmidt2011convergence}, we have that
\begin{eqnarray} \label{thm2_1.15}
 && \left \| U_{m+1} \right \|
 \\ & \leq & \nonumber  {\frac{1}{2}\sum_{k=1}^m {\lambda_k} } + \left ( A+ \left ( \frac{1}{2}\sum_{k=1}^m {\lambda_k} \right )^2 \right )^{\frac{1}{2}}
 \\ & = & \nonumber \underbrace{ \frac{1}{2}\sum_{k=1}^m {2 \gamma t_k\sqrt{{2 \frac{1}{\gamma} \varepsilon_k }}} }_{A_m} + \left (   \left \| U_{1} \right \|^2   + \underbrace{ 2 \gamma \sum_{k=1}^m  t_k^2\varepsilon_k}_{B_m} + \left ( \underbrace{ \frac{1}{2}\sum_{k=1}^m {2 \gamma t_k\sqrt{{2 \frac{1}{\gamma} \varepsilon_k }}} }_{A_m} \right )^2 \right )^{\frac{1}{2}}
 \\ & \leq & \nonumber A_m + \left ( \left \| U_{1} \right \|^2  + B_m + A_m^2  \right )^{\frac{1}{2}}
\end{eqnarray}
Because $A_k$ and $B_k$ are increasing sequences,  $\forall k \leq m$, we have that
\begin{eqnarray} \label{thm2_1.16}
\left \| U_{k} \right \|
  & \leq &  A_m + \left ( \left \| U_{1} \right \|^2 + B_m + A_m^2  \right )^{\frac{1}{2}}
    \leq   A_m +  \left \| U_{1} \right \| + \sqrt{B_m} + A_m
 \\ & \leq & \nonumber 2A_m +   \left \| U_{1} \right \| + \sqrt{B_m}
 \end{eqnarray}
 According to (\ref{thm2_1.13}) and (\ref{thm2_1.16}),  we have that
 \begin{eqnarray} \label{thm1_1.17}
 &&  t_m^2 \left  (f(x_{k+1}) - f(x^*) \right  )
\\ & \leq & \nonumber \frac{1 }{2 \gamma}  \left ( \left \| U_{1} \right \|^2 -  \left \| U_{m+1} \right \|^2 \right ) + \sum_{k=1}^m  t_k^2 \varepsilon_k + \sum_{k=1}^m t_k\sqrt{{2 \frac{1}{\gamma}\varepsilon_k }} \left \|U_{k+1}  \right \|
\\ & \leq & \nonumber \frac{1 }{2 \gamma}  \left \| U_{1} \right \|^2  + \frac{1}{2 \gamma} B_m + \frac{1}{\gamma} A_m \left ( 2A_m +   \left \| U_{1} \right \| + \sqrt{B_m} \right )
\\ & \leq & \nonumber \frac{1 }{2 \gamma}  \left ( 2A_m +   \left \| U_{1} \right \| + \sqrt{B_m} \right )^2
\end{eqnarray}
Because  $t_{k+1} = \frac{\sqrt{4t_k^2 + 1} + 1}{2}$, it is easy to verify that $t_k \geq \frac{k+1}{2}$. Thus, we have
\begin{eqnarray} \label{thm2_1.18}
f(x_{k+1}) - f(x^*)\leq \frac{2L}{(m+1)^2 } \left ( \left \| x_0 - x^* \right \|  + 2A_m + \sqrt{B_m} \right )^2
\end{eqnarray}
This completes the proof.
\end{proof}

\subsection{Experiments}\label{experiments}

\subsubsection{OSCAR}
For the  robust regression, we replace the square loss originally used in OSCAR  with the  correntropy induced loss \citep{he2011maximum}. Thus, we consider the OSCAR with the functions $g(x )$ and $h(x )$ as following.
\begin{eqnarray}\label{formulation5.1.1}
g(x ) &=&
 \frac{\sigma^2}{2}\sum_{i=1}^{l}\left (1 -e^{ - \frac{ (y_i -  X_i^Tx)^2 }{\sigma^2 }}\right )
\\  h(x ) &=& \lambda_1  \| x \|_1 +\lambda_2 \sum_{i<j} \max \{ | x_i | , |x_j  | \}
\end{eqnarray}
where  $\lambda_1\geq 0$ and $\lambda_2\geq 0$ are two  regularization parameters.  Based on the function $g(x)$, we have that $\nabla g(x) =  - \sum_{i=1}^{l} e^{ - \frac{ (y_i -  X_i^Tx)^2 }{\sigma^2 }} (y_i -  X_i^Tx) X_i$. As mentioned previously, for the each iteration of the proximal gradient  algorithm, we need to solve the (exact or inexact)  proximal operators.
%\begin{eqnarray} \label{section1_1.1}
% Q(x;x_{k-1})=  \frac{1}{2 \gamma}  \left \| x- \left (x_{t-1} - \gamma  \nabla g(x_{t-1}) \right ) \right \|^2 + h(x)
%\end{eqnarray}
\cite{zhong2012efficient}  proposed an  algorithm for exactly computing $\textrm{Prox}_{\gamma h}(x_{k-1} - \gamma  \nabla g(x_{k-1}))$. To implement IPG and AIPG efficiently, we need to compute an inexact  proximal operator $x_{k} \in \textrm{Prox}^{\varepsilon}_{\gamma h} \left (x_{k-1} - \gamma  \nabla g(x_{k-1})  \right )$, where $\varepsilon$ denotes an error in the proximal operator.

In our experiments, we proposed a subgradient algorithm (i.e., Algorithm \ref{algorithm3}) for computing  an inexact  proximal operator $x_{k} \in \textrm{Prox}^{\varepsilon}_{\gamma h} \left (x_{k-1} - \gamma  \nabla g(x_{k-1})  \right )$.  Specifically, we approximately solve the following subproblem.
\begin{eqnarray} \label{formulation5.1.2}
 Q(x;x_{k-1})=  \frac{1}{2 \gamma}  \left \| x- \left (x_{k-1} - \gamma  \nabla g(x_{k-1}) \right ) \right \|^2 + h(x)
\end{eqnarray}
There are two key procedures for  solving (\ref{formulation5.1.2})    in Algorithm \ref{algorithm3}:
 \begin{enumerate}
\item  Compute  the subgradient $\partial Q(x_t;x_{k-1})$.
\item Compute the duality gap $G(x_t;x_{k-1})=Q(x_t;x_{k-1}) - \widetilde{Q}(\alpha_t;x_{k-1})$, where $\alpha_t$ is the dual variables.
\end{enumerate}
\begin{algorithm}
\renewcommand{\algorithmicrequire}{\textbf{Input:}}
\renewcommand{\algorithmicensure}{\textbf{Output:}}
\caption{Subgradient algorithm for computing the proximal operator of OSCAR}
\begin{algorithmic}[1]
\REQUIRE   $\varepsilon$ (error), $\gamma'  $ (step size), $x_{k-1}$.
\ENSURE $x_{m}$.
\STATE  Initialize  $x_0 = x_{k-1} $, $t=0$.
\WHILE{$G(x_t;x_{k-1})=Q(x_t;x_{k-1}) - \widetilde{Q}(\alpha_t;x_{k-1}) < \varepsilon$}

 \STATE  $x_{t+1} \leftarrow  x_t - \gamma' g_t'$, where  $g_t' \in \partial Q(x_t;x_{k-1})$. %\COMMENT{Atomic writing}
 \STATE $t \leftarrow t+1$
\ENDWHILE
\end{algorithmic}
\label{algorithm3}
\end{algorithm}

\noindent \textbf{Compute  the subgradient:} \ \ Let $o(j) \in \{1,2,\cdots,N \}$ denote the order  of $| x_j |$ \footnote{Here, $x_j$ denotes the $j$-th coordinate of the vector $x$.} among $\{| x_{1} |,| x_{2} |,\cdots,| x_{N} | \}$ such that if    $o(j_1)<o(j_2) $, we have
$|x_{j_1 } | \leq  | x_{j_2 } |$. Thus, $\partial h(x) = \sum_{j=1}^{N}\left (\lambda_1 + \lambda_2 (o(j)-1)  \right  ) \partial | x_j |$.
\begin{eqnarray} \label{formulation5.1.3}
\partial Q(x;x_{k-1})= \frac{1}{ \gamma}\left (  x- \left (x_{k-1} - \gamma  \nabla g(x_{k-1}) \right ) \right ) + \partial h(x)
\end{eqnarray}

\noindent \textbf{Compute the  duality gap:} \ \
As mentioned in \cite{bach2011convex}, the dual function $\widetilde{Q}(\alpha;x_{k-1})$ of $Q(x;x_{k-1})$  can be computed as
%\begin{eqnarray}\label{formulation5.1.3.1}
%\end{eqnarray}
\begin{eqnarray}\label{formulation5.1.4}
\widetilde{Q}(\alpha;x_{k-1})= \max_{\alpha} && \frac{- \gamma }{2}\alpha^T \alpha -\alpha^T \left  (x_{k-1} - \gamma  \nabla g(x_{k-1}) \right )
\\ s.t. &&  \max_{\sum_{j=1}^{d}\left (\lambda_1 + \lambda_2 (o(j)-1)  \right  ) | x_j | =1} {\alpha^T x} \leq 1 \nonumber
\end{eqnarray}
Similar to \citep{zhong2012efficient},   the optimal $\alpha$ of $\widetilde{Q}(\alpha;x_{k-1})$ can be analytically computed as
\begin{eqnarray} \label{equation10.1}
\alpha = \min \left \{1, \frac{1}{ r^*( \nabla g( x_t))} \right \} \nabla g( x_t)
\end{eqnarray}
 where $\nabla g(x_t) =  \frac{1}{ \gamma}\left (  x_t- \left (x_{k-1} - \gamma  \nabla g(x_{k-1}) \right ) \right )$. Assume  that the indices of  $\xi$ are sorted by $|\xi_1| \leq |\xi_2| \leq \cdots  \leq |\xi_N|$, and  $r^*({\xi}) = \max_{j \in \{1,2,\cdots,N \} } \frac{\sum_{i=1}^j |\xi_i|}{\sum_{i=1}^j  \lambda_1 +(i-1) \lambda_2 }$.  Similar with \citep{zhong2012efficient}, we provide the algorithm for computing the duality gap  in Algorithm \ref{algorithm4}.
 \begin{algorithm}[htbp]
\renewcommand{\algorithmicrequire}{\textbf{Input:}}
\renewcommand{\algorithmicensure}{\textbf{Output:}}
\caption{Duality gap }
\begin{algorithmic}[1]
\REQUIRE $x_{k-1}$.
\ENSURE The duality $ \widetilde{Q}(\alpha;x_{k-1})$.

 \STATE Compute $ \xi =  \nabla g(x)$ and sort  $\xi_i$ in ascend order.

  \STATE Compute $ r^*({\xi}) $.

\STATE Compute the optimal $\alpha$ of $\widetilde{Q}(\alpha_t;x_{k-1})$ according to (\ref{equation10.1}).

\STATE Compute the duality gap $G(x_t;x_{k-1})=Q(x_t;x_{k-1}) - \widetilde{Q}(\alpha_t;x_{k-1})$.
\end{algorithmic}
\label{algorithm4}
\end{algorithm}

%\section{ Acknowledgments}
%References and End of Paper
%These lines must be placed at the end of your paper
\bibliographystyle{aaai}
\bibliography{Bibliography-File}
\end{document}